\documentclass[twoside,11pt]{article}
\usepackage{jmlr2e}
\usepackage{amsmath}
\usepackage{amssymb}
\usepackage{natbib}
\usepackage{tikz}
\usepackage{version}
\usepackage{hyperref}
\usetikzlibrary{shapes}

\begingroup
    \makeatletter
    \@for\theoremstyle:=definition,remark,plain\do{%
        \expandafter\g@addto@macro\csname th@\theoremstyle\endcsname{%
            \addtolength\thm@preskip\parskip
            }%
        }
\endgroup

\newtheorem{definition}{Definition}
\newenvironment{short-definition}[1][]{}{}
\newtheorem{axiom}[definition]{Definition}
\newtheorem{theorem}{Theorem}
\newtheorem{lemma}[theorem]{Lemma}
\newcommand{\weight}[3][]{E#1(#2,#3)}

\newcommand{\vol}[1]{v_{#1}}
\newcommand{\within}[1]{w_{#1}}
\newcommand{\volin}[2]{v_{#1}(#2)}
\newcommand{\withinin}[2]{w_{#1}(#2)}

\newcommand{\Obj}{Q}
\newcommand{\objpart}{q}

\newcommand{\RR}{\mathbb{R}}
\newcommand{\RRnonneg}{\RR_{\ge 0}}

\DeclareMathOperator{\argmax}{argmax}

\newcommand{\qmod}{\Obj_\text{modularity}}
\newcommand{\qcpm}{\Obj_\text{cpm}}
\newcommand{\qcoco}{\Obj_\text{coco}}
\newcommand{\qfixed}[1][M]{\Obj_\text{$#1$-fixed}}
\newcommand{\qext}[2]{\Obj_\text{$#1$,$#2$}}
\newcommand{\vamount}{\gamma}
\newcommand{\maxfrac}[1]{f_{#1}}
\newcommand{\indicator}{\mathbf{1}}
\newcommand{\cf}{\hat{C}}

\includeversion{GraphsAsImages}
\excludeversion{GraphsAsMatrices}

\tikzset{
  node/.style={circle,draw,inner sep=0,minimum size=5mm},
  lbl/.style={fill=white,inner sep=0.2mm,minimum size=1mm},
}

\jmlrheading{1}{2013}{1-23}{7/13}{11/13}{Twan van Laarhoven and Elena Marchiori}

\ShortHeadings{Axioms for graph clustering quality functions}{van Laarhoven and Marchiori}
\title{Axioms for graph clustering quality functions}
\firstpageno{1}

\author{%
 \name Twan van Laarhoven %
 \email tvanlaarhoven@cs.ru.nl
   \AND
 \name Elena Marchiori %
 \email elenam@cs.ru.nl \\
 \addr Institute for Computing and Information Sciences\\ Radboud University Nijmegen\\Postbus 9010\\6500 GL Nijmegen, The Netherlands %
 }
\editor{Vahab Mirrokni}

\begin{document}

\maketitle

\begin{abstract}
  We investigate properties that intuitively ought to be satisfied by graph clustering quality functions, i.e. functions that assign a score to a clustering of a graph.
  Graph clustering, also known as network community detection, is often performed by optimizing such a function.
%,which are functions from a graph and a clustering to a non-negative real number. % THIS IS FALSE
%
  Two axioms tailored for graph clustering quality functions are introduced, and the four axioms introduced in previous work on distance based clustering are reformulated and generalized for the graph setting.
  We show that modularity, a standard quality function for graph clustering, does not satisfy all of these six properties.
  %
  %Therefore variants of modularity are considered, whose analysis leads to the  introduction of a parametric quality function, called adaptive scale modularity.  This quality function depends on two parameters which can be tuned to control the clustering resolution. Standard  graph clustering quality functions, such as normalized cut and unnormalized cut, are obtained as specific cases of  adaptive scale modularity. 
  This motivates the derivation of a new family of quality functions, adaptive scale modularity, which does satisfy the proposed axioms.
  %This leads us to consider adaptive scale modularity, which does satisfy the proposed axioms.
  Adaptive scale modularity has two parameters, which give greater flexibility in the kinds of clusterings that can be found. 
  Standard  graph clustering quality functions, such as normalized cut and unnormalized cut, are obtained as special cases of adaptive scale modularity. 
  
  %
  %We analyze comparatively the resolution limit of adaptive scale modularity, that is, its tendency to prefer either small or large clusters, when varying the values of the two parameters. Results indicate its adaptability to capture the natural community structure across various types of graphs.
  %Furthermore, we show that the constant Potts model, a resolution-limit-free quality function introduced in \citep{Traag2011ResolutionLimitScope}, satisfies our axioms. 
  In general, the results of our investigation indicate that the considered axiomatic framework covers existing `good' quality functions for graph clustering, and can be used to derive an interesting new family of quality functions.
  % T: I really don't like the ``are necessary,''

  % TODO: geen referenties in abstract!

  % TODO
  %We further show how some axioms are important for implementation techniques that are used in practice.
\end{abstract}


%==============================================================================
\section{Introduction}
%==============================================================================

Following the work by \citet{Kleinberg2002impossibility} there have been various contributions to the theoretical foundation and analysis of clustering, such as axiomatic frameworks for quality functions \citep{AckermanBen-David2008axioms}, for criteria to compare clusterings \citep{Meila2005},  uniqueness theorems for specific types of clustering  \citep*{Zadeh2009uniqueness, Ackerman2013, Carlsson2013},  taxonomy of clustering  paradigms \citep{AckermanBL10},  and characterization of diversification systems \citep{gollapudi2009axiomatic}.

\Citeauthor{Kleinberg2002impossibility}
focused on clustering functions, which are functions from a distance function to a clustering. He showed that there are no clustering functions that simultaneously satisfy three intuitive properties: scale invariance, consistency and richness. \Citet{AckermanBen-David2008axioms} continued on this work, and showed that the impossibility result does not apply when formulating these properties in terms of quality functions instead of clustering functions, where consistency is replaced with a weaker property called monotonicity.

Both of these previous works are formulated in terms of distance functions over a fixed domain.
In this paper we focus on weighted graphs, where the weight of an edge indicates the strength of a connection.
%Given a distance function one can construct a graph, such as the $k$-nearest-neighbor graph \citep{}, by taking as edge weights the inverse distances.
The clustering problem on graphs is also known as network community detection. %TODO: reference here
%We investigate clustering quality functions, which are functions from a graph and a clustering to a real number. A notable example is modularity \citep{NewmanGirvan2004}. % doesn't fit here, moved to lower paragraph

Graphs provide additional freedoms over distance functions.
In particular, it is possible for two points to be unrelated, indicated by a weight of $0$.
These zero-weight edges in turn make it natural to consider graphs over different sets of nodes as part of a larger graph.
Secondly, we can allow for self loops. Self loops can indicate internal edges in a node. This notation is used for instance by \citet{Blondel2008}, where a graph is contracted based on a fine-grained clustering.

%One could view the weight of an edge as the inverse of the distance between its two nodes.
In this setting, where edges with weight $0$ are possible, \citeauthor{Kleinberg2002impossibility}'s impossibility result does not apply. This can be seen by considering the connected components of a graph. This is a graph clustering function that satisfies all three of Kleinberg's axioms: scale invariance, consistency and richness (see Section~\ref{sec:coco}).

Our focus is on the investigation of graph clustering quality functions, which are functions from a graph and a clustering to a real number `quality'. A notable example is modularity \citep{NewmanGirvan2004}.
In particular we ask which properties of quality functions intuitively ought to hold, and which are often assumed to hold when reasoning informally about graph clustering. Such properties might be called axioms for graph clustering.

The rest of this paper is organized as follows:
Section~\ref{sec:definitions-notation} gives basic definitions.
Next, section~\ref{sec:form-of-axioms} discusses different ways in which properties could be formulated.

In Section~\ref{sec:axioms} of this paper we propose an axiomatic framework that consists of six properties of graph clustering quality functions: the (adaption of) the four axioms from \citet{Kleinberg2002impossibility} and \citet{AckermanBen-David2008axioms} (permutation invariance, scale invariance, richness and monotonicity); and two additional properties specific for the graph setting (continuity and the locality).

% TODO: axioms vs properties

Then, in Section~\ref{sec:modularity},
we show that modularity does not satisfy the monotonicity and locality properties.

This result motivates the analysis of variants of modularity, leading to the derivation of a new parametric quality function in Section~\ref{sec:qext}, that satisfies all properties. This quality function, which we call adaptive scale modularity, has two parameters, $M$ and $\gamma$ which can be tuned to control the resolution of the clustering.
We show that quality functions similar to normalized cut and unnormalized cut are obtained in the limit when $M$ goes to zero and to infinity, respectively. Furthermore, setting $\gamma$ to  $0$ yields a parametric quality function similar to that proposed by \citet{Reichardt2004}.

%==============================================================================
\subsection{Related Work}\label{sec:relwork}
%==============================================================================

Previous axiomatic studies of clustering quality functions have focused mainly on hierarchical clustering and on weakest and strongest link style quality functions \citep{Kleinberg2002impossibility,AckermanBen-David2008axioms,Zadeh2009uniqueness,Carlsson2013}.  Papers in this line of work that focussed also on the partitional setting include \cite{Puzicha99theory, AckermanBBL12, AckermanBeLoSa13}. \Citet{Puzicha99theory}  investigated a particular class of clustering quality functions obtained by requiring the function to decompose into a certain additive form. \Citet{AckermanBBL12}  considered clustering in the weighted setting, in which every data point is assigned a real valued weight.  They performed a theoretical analysis on the influence of weighted data on standard clustering algorithms. \Citet{AckermanBeLoSa13}  analyzed robustness of clustering algorithms to the addition of a small set of points, and investigated  the robustness of popular clustering methods.

All these studies are framed in terms of distance (or similarity and dissimilarity) functions.
%This style of clustering is not commonly used for network community detection (see e.g. an overview in \citealp{Fortunato2010}).

\Citet{Bubeck2009} studied statistical consistency of clustering methods. They introduced the so-called nearest neighbor clustering and showed its consistency also for standard  graph based quality functions, such as normalized cut, ratio cut, and modularity. Here we do not focus on properties of methods to optimize clustering quality, but on natural properties that quality functions for graph clustering should satisfy.

Related works on graph clustering quality functions mainly focus on the so-called resolution limit, that is, the tendency of a quality function to prefer either small or large clusters.
In particular, \citet{Fortunato2007ResolutionLimit} proved that modularity may not detect clusters smaller than a scale which depends on the total size of the network and on the degree of interconnectedness of the clusters.  \Citet{Laarhoven2013} showed that the resolution limit is the most important difference between quality functions in graph clustering optimized using local search optimization.
%Here we analyze empirically the resolution behavior of adaptive modularity when varying the values of its parameters. 

To mitigate the resolution limit phenomenon, the quality function may be extended with a so-called resolution parameter.  For example, \Citet{Reichardt2006statistical} proposed a formulation of graph clustering (therein called network community detection)  based on principles from statistical mechanics. This interpretation leads to the introduction of a family of quality functions with a parameter that allows to control the clustering resolution. 
%, namely finding the ground state of an infinite range spin glass.  This interpretation leads to the introduction of a family of quality functions which includes an extension of modularity with a parameter that allows to control the clustering resolution.
In Section~\ref{sec:qext-related} we will show that this extension is a special case of adaptive scale modularity.

%In \cite{} it was shown that in general single global optimization criteria do not seem capable for detecting all communities if their size distribution is broad. We analyze comparatively the resolution limit of the new quality functions derived from the proposed axiomatization, and show its relation with popular quality functions used in graph clustering with respect to the resolution limit.

% Dit komt later wel.
\Citet*{Traag2011ResolutionLimitScope} formalized the notion of resolution-free quality functions, that is, not suffering from the resolution limit, and  provided a characterization of this class of quality functions. Their notion is essentially an axiom, and we will discuss the relation to our axioms in Section~\ref{sec:rlf}.
% \Citet{Traag2011ResolutionLimitScope}  formalized the notion of resolution-free quality functions, that is, not suffering from the resolution limit, and provided a characterization of this class of quality functions. The general idea is that when looking at any induced subgraph of the original graph, the clustering of that graph should not change. In particular, the authors prove that a clustering quality function is resolution-free if the weights assigned to edges are local, that is, for each subgraph, they are equal to the weights in that subgraph time a scale factor that depends on the subgraph. 
% The parameter $\gamma$  scales the weights of each cluster. 
% The relation between the locality axiom and the definition of resolution-free quality functions is discussed in section \ref{sec:rlf}).
% In particular, we show that adaptive scale modularity is not resolution-limit-free and that the constant Potts model, a resolution-limit-free quality function introduced in \citep{Traag2011ResolutionLimitScope}, satisfies all axioms.

%==============================================================================
\section{Definitions and Notation}\label{sec:definitions-notation}
%==============================================================================

\begin{short-definition}[Graph]
  A \emph{symmetric weighted graph} is a pair $(V,E)$ of a finite set $V$ of nodes and a function $E : V \times V \to \RRnonneg$ of edge weights, where $E(i,j) = E(j,i)$ for all $i,j \in V$.%
\end{short-definition}
Edges with larger weights represent stronger connections, so missing edges can get weight $0$.
Note that this is the opposite of the convention used in distance based clustering.
%In this paper we consider symmetric graphs, where $\weight{i}{j} = \weight{j}{i}$ for all nodes $i,j \in V$.
We explicitly allow for self loops, that is, nodes for which $\weight{i}{i} > 0$.
%
%We will sometimes write a graph with nodes $\{1,2,\dotsc,n\}$ as its $n \times n$ adjacency matrix.

\begin{short-definition}[Clustering]
  A \emph{clustering} $C$ of a graph $G=(V,E)$ is a partition of its nodes.
  That is, $\bigcup C = V$ and for all $c_1,c_2 \in C$, $c_1 \cap c_2 \neq \emptyset$ if and only if $c_1 = c_2$.
\end{short-definition}
%When two nodes $i$ and $j$ are in the same cluster in clustering $C$, i.e. when $\exists{c \in C}. i,j \in c$, then we write $i \sim_C j$. Otherwise we write $i \not\sim_C j$.
When two nodes $i$ and $j$ are in the same cluster in clustering $C$, i.e. when $i,j \in c$ for some $c \in C$, then we write $i \sim_C j$. Otherwise we write $i \not\sim_C j$.

\begin{short-definition}[Refinement]
  A clustering $C$ is a \emph{refinement} of a clustering $D$, written $C \sqsubseteq D$, when
  for every cluster $c \in C$ there is a cluster $d \in D$ such that $c \subseteq d$.
  % every cluster in $C$ is a subset of some cluster in $D$.
\end{short-definition}

% % Moved to section on modularity.
% \begin{short-definition}[Volume]permutation-invariance
%   The \emph{volume} of a cluster $c$ in a graph is
%   $v_c = \sum_{i \in c}\sum_{j \in v} \weight{i}{j}$.
%   The volume of the entire graph can then be written as $v_V$.
% \end{short-definition}
% \begin{short-definition}
%   The \emph{within cluster weight} of a cluster $c$ is $w_c = \sum_{i,j \in c} \weight{i}{j}$.
% \end{short-definition}
% \begin{short-definition}
%   The \emph{between cluster weight} is $b_c = v_c - w_c$.
% \end{short-definition}

\begin{short-definition}[Quality function]
  %A \emph{graph clustering quality function} $\Obj$ is a function from graphs $G$ and clusterings of $G$ to a partially ordered set $P$.
  %For simplicity one can take $P$ to be the real numbers.
  A \emph{graph clustering quality function} (or objective function) $\Obj$ is a function from graphs $G$ and clusterings of $G$ to
  real numbers.
  % $\RR$.
\end{short-definition}
We adopt the convention that a higher quality indicates a `better' clustering.
As a generalization, we will sometimes work with parameterized \emph{families of quality functions}. A single quality function can be seen as a family with no parameters.
% \begin{short-definition}[Family]
%   %A \emph{graph clustering quality function} $\Obj$ is a function from graphs $G$ and clusterings of $G$ to a partially ordered set $P$.
%   %For simplicity one can take $P$ to be the real numbers.
%   A \emph{graph clustering quality function} $\Obj$ is a function from graphs $G$ and clusterings of $G$ to
%   the real numbers.
%   % $\RR$.
% \end{short-definition}

\begin{short-definition}[Agreement]
  Let $G_1 = (V_1, E_1)$ and $G_2 = (V_2, E_2)$ be two graphs and let $V_a \subseteq V_1 \cap V_2$ be a subset of the common nodes.
  %We say that they are \emph{consistent on a set} $V_a \subseteq V_1 \cap V_2$
  We say that the graphs \emph{agree on} $V_a$ if $E_1(i,j) = E_2(i,j)$ for all $i,j \in V_a$.
  %
  %If, furthermore, both the $G_1$-neighborhood and the $G_2$-neighborhood of $V_a$ are contained in $V_1 \cap V_2$
  %then we say that the graphs are \emph{fully consistent on $V_a$}.
  %We say that the graphs \emph{agree on the neighborhood of $V_a$} if the $G_1$-neighborhood and the $G_2$-neighborhood are equal and contained in $V_1 \cap V_2$.
  %
  %We say that the graphs also \emph{agree on the neighborhood of $V_a$}
  %if $E_1(i,j) = E_2(i,j)$ whenever $i \in V_a$ \emph{or} $j \in V_a$. Here we take $E_1(i,j)=0$ and $E_1(i,j)=0$ for nodes not in the respective graphs.
  % In particular this means that there can be no edges from nodes in $V_a$ to nodes not in $V_1 \cap V_2$.
  %We say that the graphs \emph{agree on $V_a$ and its neighborhood} if 
  We say that the graphs also \emph{agree on the neighborhood of $V_a$}
  If \begin{itemize}
    \item $E_1(i,j) = E_2(i,j)$ for all $i \in V_a$ and $j \in V_1 \cap V_2$,
    \item $E_1(i,j) = 0$ for all $i \in V_a$ and $j \in V_1 \setminus V_2$, and
    \item $E_2(i,j) = 0$ for all $i \in V_a$ and $j \in V_2 \setminus V_1$.\end{itemize}
%   If $E_1(i,j) = E_2(i,j)$ for all $i \in V_a$ and $j \in V_1 \cap V_2$,
%      $E_1(i,j) = 0$ for all $i \in V_a$ and $j \in V_1 \setminus V_2$, and
%      $E_2(i,j) = 0$ for all $i \in V_a$ and $j \in V_2 \setminus V_1$.
%   
\end{short-definition}
This means that for nodes in $V_a$ the weights and endpoints of incident edges are exactly the same in the two graphs.

%}

% \begin{short-definition}[Isomorphism of a graph]
%   Isomorphisms between node sets can be extended to graphs and clusterings.
%   If $f : V \to V'$ is an isomorphism then 
%   $f(C) = \{\{f(i)\mid i \in c\} \mid c\in C\}$ and
%   $f(V,E) = (V',(i,j)\mapsto E(f^{-1}(i),f^{-1}(j)))$.
% \end{short-definition}

%
% If, furthermore %, there is a set $V_a \subseteq V_1 \cap V_2$ such that
%     $E_1(i,j)=0$ for all $i \in V_a$ and $j \in V_1 \setminus V_2$
% and $E_2(i,j)=0$ for all $i \in V_a$ and $j \in V_2 \setminus V_1$, then we say that the graphs are \emph{fully consistent on $V_a$}.

%==============================================================================
%\section{On the form of axioms}
%==============================================================================

%
%%------------------------------------------------------------------------------
\section{On the Form of Axioms}\label{sec:form-of-axioms}
%%------------------------------------------------------------------------------

There are three different ways to state potential axioms for clustering:
%
%Many axioms can be stated in three variants:
\begin{enumerate}
\item As a property of clustering functions, as in \cite{Kleinberg2002impossibility}.
      For example, scale invariance of a clustering function $\cf$ would be written as
      ``$\cf(G) = \cf(\alpha G)$, for all graphs $G$, $\alpha>0$''.
      I.e. the optimal clustering is invariant under scaling of edge weights.
\item As a property of the values of a quality function $\Obj$, as in \Citet{AckermanBen-David2008axioms}.
      For example ``$\Obj(G,C) = \Obj(\alpha G,C)$, for all graphs $G$, all clustering $C$ of $G$, and $\alpha>0$''.
      I.e. the quality is invariant under scaling of edge weights.
\item As a property of the relation between qualities of different clustering, or equivalently, as a property of an ordering of clusterings for a particular graph.
      For example
      ``$\Obj(G,C) \ge \Obj(G,D) \Rightarrow \Obj(\alpha G,C) \ge \Obj(\alpha G,D)$''.%, for all graphs $G$ and all clustering $C$ and $D$ of $G$.
      I.e. the `better than' relation for clusterings is invariant under scaling of edge weights.
\end{enumerate}

The third form is slightly more flexible than the other two. Any quality function that satisfies a property in the second style will also satisfy the corresponding property in the third style, but the converse is not true.
%An axiom in the absolute style makes a specific choice about the value of the objec
Note also that if $D$ is not restricted in a property in the third style, then one can take $\cf(G) = \argmax_C \Obj(G,C)$ to obtain a clustering function and an axiom in the first style.

Most properties are more easily stated and proved in the second, absolute, style.
Therefore, we adopt the second style unless doing so requires us to make specific choices.

% TODO: say this only once
%When the goal of a quality function is to optimize the clustering of a single graph, the relation between qualities on different graph is not important.
%The second form often requires an explicit choice. For example the scale invariance axiom of \citeauthor{Ackerman} requires that the quality is equal for scaled graph.
%An axiom for consistency under unions could require that the quality for the combined graphs is the sum (or some other function) of the quality on the two parts.
% TODO: 
%For example, the continuity axiom makes no sense when only considering relations.
%, such as the scale invariance and consistency under unions axioms.
%In the third case a quality function can instead be replaced by a total order on clusterings of a given graph. Write C ≤_q^G D when \Obj(C,G) ≤ \Obj(...TODO... ≤_q^G ⇒ ≤_q^G'

%==============================================================================
\section{Axioms for Graph Clustering Quality Functions}\label{sec:axioms}
%==============================================================================

\Citeauthor{Kleinberg2002impossibility} defined three axioms for distance based clustering functions.
In \citet{AckermanBen-David2008axioms} the authors reformulated these into four axioms for clustering quality functions.
These axioms can easily be adapted to the graph setting.
%Here we restate these axioms

%-----

The first property that one expects for graph clustering is that the quality of a clustering depends only on the graph, that is, only on the weight of edges between nodes, not on the identity of nodes. We formalize this in the permutation invariance axiom,
%The first axiom, permutation invariance, requires that a quality function does not depend on the identity of nodes, but only on the weight of edges between them. We might also call this permutation invariance.

\begin{axiom}[Permutation invariance]
  \label{ax:permutation-invariance}
  A graph clustering quality function $\Obj$ is \emph{permutation invariant} if 
  %for all permutations $\pi$, $\Obj(G,C) = \Obj(\pi(G),\pi(C))$.
  for all graphs $G=(V,E)$ and all isomorphisms $f : V \to V'$,
  it is the case that $\Obj(G,C) = \Obj(f(G),f(C))$;
  where $f$ is extended to graphs and clusterings by
  $f(C) = \{\{f(i)\mid i \in c\} \mid c\in C\}$ and
  $f((V,E)) = (V',(i,j)\mapsto E(f^{-1}(i),f^{-1}(j)))$.
\end{axiom}

%-----

The second property, scale invariance, requires that the quality doesn't change when edge weights are scaled uniformly.
This is an intuitive axiom when one thinks in terms of units: a graph with edges in ``m/s'' can be scaled to a graph with edges in ``km/h''. The quality should not be affected by such a transformation, perhaps up to a change in units.

\citet{AckermanBen-David2008axioms} defined scale invariance by insisting that the quality stays equal when distances are scaled.
In contrast, in \citet{Puzicha99theory} the quality should scale proportional with the scaling of distances.
We generalize both of these previous definitions by only considering the relations between the quality of two clusterings.

\begin{axiom}[Scale invariance]
  \label{ax:scale-invariance}
  A graph clustering quality function $\Obj$ is \emph{scale invariant} if
  for all graphs $G=(V,E)$,
  all clusterings $C_1,C_2$ of $G$
  and all constants $\alpha > 0$,
  $\Obj(G,C_1) \leq \Obj(G,C_2)$ if and only if $\Obj(\alpha G,C_1) \leq \Obj(\alpha G,C_2)$.
  Where $\alpha G=(V,(i,j) \mapsto \alpha E(i,j))$ is a graph with edge weights scaled by a factor $\alpha$.
\end{axiom}

This formulation is flexible enough for single quality functions. However, families of quality functions could have parameters that are also scale dependent. For such families we therefore propose to use as an axiom a more flexible property that also allows the parameters to be scaled,
% However, often quality functions have extra parameters, and these parameters might also be scale dependent. For such families we therefore propose to use a more flexible axiom,

% %Another consideration is that quality functions often have extra parameters, such as a thresholds or target scales. These parameters might also be scale dependent. For such \emph{families of quality functions} we therefore propose a more flexible scale invariance axiom. Single quality functions can be seen as a family with no parameters.
% Another consideration is that families of quality functions often have extra parameters, such as a threshold or target scales, which are also scale dependent. To accommodate such families, we allow the parameter values to change with the scaling of edge weights.
% %
% This makes intuitive sense when one thinks in terms of units: a graph with edges in ``m/s'' can be scaled to a graph with edges in ``km/h'', but then the parameters of a quality function might also have to be scaled from ``m/s'' to ``km/h''.

\begin{axiom}[Scale invariant family]
  \label{ax:scale-invariant-family}
  \def\param{P}
  A family of quality function $\Obj_\param$ parameterized by $\param \in \mathcal{\param}$ is \emph{scale invariant} if
  for all constants $\param \in \mathcal{\param}$ and $\alpha > 0$
  there is a $\param' \in \mathcal{\param}$
  such that
  for all graphs $G=(V,E)$,
  and all clusterings $C_1,C_2$ of $G$,
  %$\Obj_M(G,C_1) \leq \Obj_M(G,C_2)$ if and only if $\Obj_{\alpha M}(\alpha G,C_1) \leq \Obj_{\alpha M}(\alpha G,C_2)$.
  $\Obj_\param(G,C_1) \leq \Obj_\param(G,C_2)$ if and only if $\Obj_{\param'}(\alpha G,C_1) \leq \Obj_{\param'}(\alpha G,C_2)$.
  %Where $\alpha G=(V,(i,j) \mapsto \alpha E(i,j))$ is a graph with edge weights scaled by a factor $\alpha$.
\end{axiom}

%-----

Thirdly, we want to rule out trivial quality functions. This is done by requiring richness, i.e. that by changing the edge weights any clustering can be made optimal for that quality function.

\begin{axiom}[Richness]
  \label{ax:richness}
  A graph clustering quality function $\Obj$ is \emph{rich} if
  for all sets $V$ and all non-trivial partitions $C^*$ of $V$, there is a graph $G=(V,E)$ such that $C^*$ is the $Q$-optimal clustering of $V$, i.e. $\argmax_{C} \Obj(G,C) = C^*$.
\end{axiom}

%-----

The last axiom that \citeauthor{AckermanBen-David2008axioms} consider is by far the most interesting.
Intuitively, we expect that when the edges within a cluster are strengthened, or when edges between clusters are weakened, that this does not decrease the quality. Formally we call such a change of a graph a consistent improvement,

\begin{definition}[Consistent improvement]
  Let $G=(V,E)$ be a graph and $C$ a clustering of $G$.
  A graph $G'=(V,E')$ is a \emph{$C$-consistent improvement of $G$} if
  for all nodes  $i$ and $j$,
  $\weight[']{i}{j} \ge \weight{i}{j}$ whenever $i \sim_C j$ and
  $\weight[']{i}{j} \le \weight{i}{j}$ whenever $i \not\sim_C j$.
\end{definition}

We say that a quality function that does not decrease under consistent improvement is monotonic. In previous work this axiom is often called consistency.

\begin{axiom}[Monotonicity]
  \label{ax:monotonicity}
  A graph clustering quality function $\Obj$ is \emph{monotonic} if
  for all graphs $G$, all clusterings $C$ of $G$ and all $C$-consistent improvements $G'$ of $G$
  it is the case that
  $\Obj(G',C) \ge \Obj(G,C)$.
\end{axiom}

\subsection{Locality}\label{sec:locality}
%------------------------------------------------------------------------------

In the graph setting it also becomes natural to look at combining different graphs. With distance functions this is impossible, since it is not clear what the distance between nodes from the two different sets should be. But for graphs we can take the edge weight between nodes not in both graphs to be zero, which is the case when the graphs agree on the neighborhood of some set.

%When combining two disjoint graphs, 
%The monotonicity axiom deals with cases where edge weights in a graph change, but it does not deal 
Consider adding nodes to one side of a large network, then we would not want the clustering on the other side of the network to change if there is no direct connection.
%For example if a new person joins a social network in China, then the optimal cluster of a person in Europe should remain the same.
For example, if a new protein is discovered in yeast, then the clustering of unrelated proteins in humans should remain the same.
Similarly, we can consider any two graphs with disjoint node sets as one larger graph. Then the quality of clusterings of the two original graphs should relate directly to quality on the combined graph.

In general, local changes to a graph should have only local consequences to a clustering. Or in other words, the contribution of a single cluster to the total quality should only depend on nodes in the neighborhood of that cluster.

\begin{axiom}[Locality]
  \label{ax:locality}
  A graph clustering quality function $\Obj$ is \emph{local} if
  for all graphs $G_1 = (V_1, E_1)$ and $G_2 = (V_2, E_2)$ that agree on a set $V_a$ and its neighborhood,
  and for all clusterings $C_a,D_a$ of $V_a$, $C_1$ of $V_1 \setminus V_a$ and $C_2$ of $V_2 \setminus V_a$,
  if   $\Obj(G_1, C_a \cup C_1) \ge \Obj(G_1, D_a \cup C_1)$
  then $\Obj(G_2, C_a \cup C_2) \ge \Obj(G_2, D_a \cup C_2)$.
\end{axiom}

%This axiom rules out quality functions with a strong resolution bias. % T: No it does not
Any quality function that has a preference for a fixed number of clusters will not be local.
On the other hand, a quality function that is written as a sum over clusters, where each summand depends only on properties of nodes and edges in one cluster and not on global properties, is local.

%------------------------------------------------------------------------------
%\subsubsection{Relation \citet{AckermanBenDavidLokerCOLT2010}}\label{sec:rlf}
%------------------------------------------------------------------------------

\Citet{AckermanBenDavidLokerCOLT2010} defined a similar locality property for clustering functions.
Their definition differs from ours in three ways.
First of all, they looked at $k$-clustering, where the number of clusters is given and fixed.
Secondly, their locality property only implies a consistent clustering when the rest of the graph is removed, corresponding to $V_2 = V_1\cap V_a$. They do not consider the other direction, where more nodes and edges are added.
Finally, their locality property requires only agreement of the overlapping set $V_a$, not on its neighborhood. That means that clustering functions should also give the same results if edges with one endpoint in $V_a$ are removed.

%------------------------------------------------------------------------------
\subsubsection{Relation to Resolution-Limit-Free Quality Functions}\label{sec:rlf}
%------------------------------------------------------------------------------

\Citet{Traag2011ResolutionLimitScope} introduced the notion of \emph{resolution-limit-free} quality functions, which is similar to locality. They then showed that resolution-limit-free quality functions do not suffer from the resolution limit as described by \citet{Fortunato2007ResolutionLimit}. Their definition is as follows.

\begin{definition}[Resolution-limit-free]
 \label{ax:rlf}
 Call a clustering $C$ of a graph $G$ $\Obj$-optimal if for all clustering $C'$ of $G$ we have that  $\Obj(G,C) \geq \Obj(G,C')$. Let $C$  be a $\Obj$-optimal clustering of a graph $G_1$. Then the quality function $\Obj$ is
called resolution-limit-free if for each subgraph $G_2$ induced by $D \subset C$, the partition $D$ is also $\Obj$-optimal.
\end{definition}

There are three differences compared to our locality property. First of all, Definition~\ref{ax:rlf} refers only to the optimal clustering, not to the quality, i.e. it is a property in the style of Kleinberg.
Secondly, locality does not require that $G_2$ be a subgraph of $G_1$. Locality is stronger in that sense.
Thirdly, and perhaps most importantly, in the subgraph $G_2$ induced by $D \subset C$, edges from a node in $D$ to nodes not in $D$ will be removed. That means that while $G_1$ and $G_2$ agree on the set of common nodes, they do not also agree on their neighborhood. So in this sense locality is weaker than resolution-limit-freedom.

The notion of resolution-limit-free quality functions was born out of the need to avoid the resolution limit of graph clustering. And indeed locality is not enough to guarantee that a quality function is free from this resolution limit.

We could look at a stronger version of locality, which replaces agreement on the neighborhood of a set $V_a$ by plain agreement on that set. Such a \emph{strong locality} property would imply resolution-limit-freedom. However, it is a very strong property in that it rules out many sensible quality functions. In particular, a strongly local quality function can not depend on the weight of edges entering or leaving a cluster, because that weight can be different in another graph that agrees only on that cluster.

The solution used by \citeauthor{Traag2011ResolutionLimitScope} is to use the number of nodes instead of the volume of a cluster. In this way they obtain a resolution-limit-free variant of the Potts model by \citet{Reichardt2004}, which they call the constant Potts model. But this comes at the cost of scale invariance.

\subsection{Continuity}\label{sec:coco}
%------------------------------------------------------------------------------

In the context of graphs, perhaps the most intuitive clustering function is finding the connected components of a graph.
As a quality function, we could write
\begin{equation*}
  %\qcoco(G,C) = \indicator[\text{$C$ contains exactly the connected components of $G$}].
  \qcoco(G,C) = \indicator[C = \cf_\text{coco}(G)],
\end{equation*}
where the function $\cf_\text{coco}$ yields the connected components of a graph.

This quality function is clearly permutation invariant, scale invariant, rich, and local.
% ContinuousTODO: mention other axioms (subject to their order in the paper).
Since a consistent change can only remove edges between clusters and add edges within clusters, the coco quality function is also monotonic.

In fact, all of Kleinberg's axioms (reformulated in terms of graphs) also hold for $\cf_\text{coco}$, which seems to refute their impossibility result.
However, the impossibility proof can not be directly transfered to graphs,
because it involves a multiplication and division by a maximum distance. In the graph setting this would be multiplication and division by a minimum edge weight, which can be zero.

Still, despite connected components satisfying all previously defined properties (except for strong locality), it is not a very useful quality function.
In many real-world graphs, most nodes are part of one giant connected component \citep{Bollobas2001GiantCoCo}.
We would also like the clustering to be influenced by the weight of edges, not just by their existence.
%One way to capture this notion, and also to rule out other degenerate quality functions
A natural way to rule out such degenerate quality functions is to require continuity.

\begin{axiom}[Continuity]
  \label{ax:continuity}
  A quality function $\Obj$ is \emph{continuous} if a small change in the graph leads to a small change in the quality.
  Formally, $\Obj$ is continuous if for every $\epsilon > 0$ and every graph $G=(V,E)$ there exists a $\delta > 0$ such that for all graphs $G'=(V,E')$, if $E(i,j) - \delta < E'(i,j) < E(i,j) + \delta$ for all nodes $i$ and $j$,
  then $\Obj(G',C) - \epsilon < \Obj(G,C) < \Obj(G',C) + \epsilon$ for all clusterings $C$ of $G$.
  %Formally, $\Obj$ is continuous if for every $\epsilon > 0$ and every set $V$ there exists a $\delta > 0$ such that for all graphs $G=(V,E)$ and $G'=(V,E')$, if $E(i,j) - \delta < E'(i,j) < E(i,j) + \delta$ for all nodes $i$ and $j$,
  %then $\Obj(G',C) - \epsilon < \Obj(G,C) < \Obj(G',C) + \epsilon$ for all clusterings $C$.
\end{axiom}
%TODO: I could just say that $\Obj$ is a continuous function in its first argument with respect to the metric $d(G,G') = \sum{i,j} |\weight{i}{j} - \weight[']{i}{j}|^2$.

Connected components clustering is not continuous, because adding an edge with a small weight $\delta$ between clusters changes the connected components, and hence dramatically changes the quality.

Continuous quality functions have an important property in practice, in that they provide a degree of robustness to noise. A clustering that is optimal with regard to a continuous quality function will still be close to optimal after a small change to the graph.

%------------------------------------------------------------------------------
\subsection{Summary of Axioms}
%------------------------------------------------------------------------------

We propose to consider the following six properties as axioms for graph clustering quality functions,
\begin{enumerate}
  \item Permutation invariance (definition~\ref{ax:permutation-invariance}),
  \item Scale invariance (definition~\ref{ax:scale-invariance}),
  \item Richness (definition~\ref{ax:richness}),
  \item Monotonicity (definition~\ref{ax:monotonicity}),
  \item Locality (definition~\ref{ax:locality}), and
  \item Continuity (definition~\ref{ax:continuity}).
\end{enumerate}
As mentioned previously, for families of quality functions we replace scale invariance by scale invariance for families (definition~\ref{ax:scale-invariant-family}).

In the next section we will show that this set of axioms is consistent by defining a quality function and a family of quality functions that satisfies all of them.
Additionally, the fact that there are quality functions that satisfy only some of the axioms shows that they are (at least partially) independent.

% TODO: Are all these axioms necessary? That is, if we exclude any one of these axioms,  are there functions that do not make good clustering-quality measures but satisfy the remaining axioms? We showed that this is the case for continuity. But what about locality?

%------------------------------------------------------------------------------
%\subsection{Soundness}
%------------------------------------------------------------------------------

%We will give a quality function that satisfies all six axioms in section~\ref{sec:qext}, thereby showing that this set is sound.
%TODO: meer?

%==============================================================================
\section{Modularity}\label{sec:modularity}
%==============================================================================

For graph clustering one of the most popular quality functions is modularity \citep{NewmanGirvan2004}, despite its limitations \citep{Good2010,Traag2011ResolutionLimitScope},
\begin{equation}
  \qmod(G,C) = \sum_{c \in C}\biggl( \frac{\within{c}}{\vol{V}} - \Bigl( \frac{\vol{c}}{\vol{V}} \Bigr)^2 \biggr)
  .
\end{equation}
In this expression $\volin{c}{G} = \sum_{i \in c}\sum_{j \in V} \weight{i}{j}$ is the volume of a cluster,
while $\withinin{c}{G} = \sum_{i,j \in c} \weight{i}{j}$ is the within cluster weight. $v_V$ is the volume of the entire graph. We leave the argument $G$ implicit for readability.
%
%Note that these quantities depend on the graph, but the argument $G$ is implicit.

It is easy to see that modularity is permutation invariant, scale invariant and continuous.

\begin{theorem}
  Modularity is rich.
  % if you exclude trivial clusterings
  \label{thm:modularity-is-rich}
\end{theorem}
%\begin{proof}
%  \input{proof-modularity-rich.tex}
%\end{proof}
The proof of Theorem~\ref{thm:modularity-is-rich} is in appendix~\ref{sec:proof-modularity-rich}.

An important aspect of modularity is that volume and within weight are normalized with respect to the total volume of the graph. This ensures that the quality function is scale invariant, but it also means that the quality can change in unexpected ways when the total volume of the graph changes.
%This leads us to theorem~\ref{thm:modularity-no-monotonicity}.
This leads us to Theorem~\ref{thm:modularity-not-local}.

% \begin{theorem}
%   Modularity is not consistent under unions.
%   \label{thm:modularity-no-union}
% \end{theorem}
% \begin{proof}
%   Consider the graph
%   \begin{align*}
%       G = \begin{tikzpicture}
%             [baseline=-1mm]
%             \draw[use as bounding box] (-0.45,-0.7) rectangle (4.95,0.7);
%             \node[node](a) at (0,0) {a};
%             \node[node](b) at (1.5,0) {b};
%             \node[node](c) at (3,0) {c};
%             \node[node](d) at (4.5,0) {d};
%             \draw[bend left=15,->] (a) to node[above] {1} (b);
%             \draw[bend left=15,->] (b) to node[below] {1} (a);
%             \draw[bend left=15,->] (b) to node[above] {1} (c);
%             \draw[bend left=15,->] (c) to node[below] {1} (b);
%             \draw[bend left=15,->] (c) to node[above] {1} (d);
%             \draw[bend left=15,->] (d) to node[below] {1} (c);
%           \end{tikzpicture},
%   \end{align*}
%   and the clusterings $C=\{\{a,b,c,d\}\}$ and $D=\{\{a,b\},\{c,d\}\}$.
%   The modularity of $C$ is $0$, the modularity of $D$ is $1/6$.
%   %
%   %
%   %Consider a graph with four nodes connected in a chain, as shown in figure~\ref{fig:modularity-union} (left).
%   %A clustering $C$ where all nodes are in a single cluster has modularity $0$,
%   %compared to a clustering $D$ with two clusters, which has modularity $1/6$.
%   Now take the disjoint union of two copies of this graph.
%   The clustering $C\disjointunion C$ will have modularity $1/2$,
%   while $D\disjointunion D$ has modularity $5/12$.
%   This counterexample shows that modularity is not consistent under unions.
% \end{proof}

\begin{theorem}
  Modularity is not local.
  \label{thm:modularity-not-local}
\end{theorem}
\begin{proof}
  Consider the graphs
  \begin{align*}
      G_1 = \begin{tikzpicture}
            [baseline=+2mm]
            \draw[use as bounding box] (-0.55,-0.7) rectangle (3.65-1.5,0.7+0.7);
            \node[node](a) at (0,0) {a};
            \node[node](b) at (1.5,0) {b};
            \draw[bend left=15,->] (a) to node[above] {1} (b);
            \draw[bend left=15,->] (b) to node[below] {1} (a);
            \draw[loop,->] (a) to node[above] {2} (a);
            \draw[loop,->] (b) to node[above] {2} (b);
          \end{tikzpicture}&&
      G_2 = \begin{tikzpicture}
            [baseline=+2mm]
            \draw[use as bounding box] (-0.55,-0.7) rectangle (3.65,0.7+0.7);
            \node[node](a) at (0,0) {a};
            \node[node](b) at (1.5,0) {b};
            \node[node](c) at (3,0) {c};
            \draw[bend left=15,->] (a) to node[above] {1} (b);
            \draw[bend left=15,->] (b) to node[below] {1} (a);
            \draw[loop,->] (a) to node[above] {2} (a);
            \draw[loop,->] (b) to node[above] {2} (b);
            \draw[loop,->] (c) to node[above] {4} (c);
          \end{tikzpicture},
  \end{align*}
  which agree on the set $V_a = \{a,b\}$.
  Note that we draw the graphs as directed graphs, to make it clear that each undirected edge is counted twice for the purposes of volume and within cluster weight.
  Now take the clusterings $C_a=\{\{a\},\{b\}\}$ and $D_a=\{\{a,b\}\}$ of $V_a$; $C_1=\{\}$ of $V_1\setminus V_a$; and $C_2=\{\{c\}\}$ of $V_2\setminus V_a$.
  Then \[\qmod(G_1,C_a\cup C_1) = 1/6 > 0 = \qmod(G_1,D_a\cup C_1),\]
  while \[\qmod(G_2,C_a\cup C_2) = 23/50 < 24/50 = \qmod(G_2,D_a\cup C_2).\]
  This counterexample shows that modularity is not local.
\end{proof}

Even without changing the node set, changes in the total volume can be problematic, as shown by the following theorem.

\begin{theorem}
  Modularity is not monotonic.
  \label{thm:modularity-no-monotonicity}
\end{theorem}
\begin{proof}
  \begin{GraphsAsMatrices}
    Consider the graphs
    \begin{align*}
      G = \begin{bmatrix}
            0 & 1 & 0 & 0 \\
            1 & 0 & 0 & 0 \\
            0 & 0 & 0 & 1 \\
            0 & 0 & 1 & 0 \\
          \end{bmatrix}
      &&
      G' = \begin{bmatrix}
            0 & 0 & 0 & 0 \\
            0 & 0 & 0 & 0 \\
            0 & 0 & 0 & 1 \\
            0 & 0 & 1 & 0
          \end{bmatrix},
    \end{align*}
    and the clustering $C = \{\{1\},\{2\},\{3,4\}\}$.%
  \end{GraphsAsMatrices}
  \begin{GraphsAsImages}
    Consider the graphs
    \begin{align*}
%       G = \begin{tikzpicture}
%             [baseline=-7mm]
%             %\useasboundingbox (0,0) rectangle (1.2,-1);
%             \node[node](a) at (0,0) {a};
%             \node[node](b) at (1.5,0) {b};
%             \node[node](c) at (0.75,-1) {c};
%             %\node[node](d) at (1.2,-1) {d};
%             \draw[bend left=15,->] (a) to node[above] {1} (b);
%             \draw[bend left=15,->] (b) to node[below] {1} (a);
%             %\draw (c) to node[below] {1} (d);
%             \draw[loop right,->] (c) to node[right] {2} (c);
%           \end{tikzpicture}
%       &&
%       G' = \begin{tikzpicture}
%             [baseline=-7mm]
%             %\useasboundingbox (0,0) rectangle (1.2,-1);
%             \node[node](a) at (0,0) {a};
%             \node[node](b) at (1.5,0) {b};
%             \node[node](c) at (0.75,-1) {c};
%             %\node[node](d) at (1.2,-1) {d};
%             \draw[bend left=15,->] (a) to node[above] {0} (b);
%             \draw[bend left=15,->] (b) to node[below] {0} (a);
%             %\draw (c) to node[below] {1} (d);
%             \draw[loop right,->] (c) to node[right] {2} (c);
%           \end{tikzpicture},
      G = \begin{tikzpicture}
            [baseline=3mm]
            \draw[use as bounding box] (-0.4,-0.4) rectangle (3.1,1.3);
            \node[node](a) at (0,0.4) {a};
            \node[node](b) at (1.5,0.4) {b};
            \node[node](c) at (2.5,0) {c};
            %\node[node](d) at (1.2,-1) {d};
            \draw[bend left=15,->] (a) to node[above] {1} (b);
            \draw[bend left=15,->] (b) to node[below] {1} (a);
            %\draw (c) to node[below] {1} (d);
            \draw[loop,->] (c) to node[above] {2} (c);
          \end{tikzpicture}
      &&
      G' = \begin{tikzpicture}
            [baseline=3mm]
            \draw[use as bounding box] (-0.4,-0.4) rectangle (3.1,1.3);
            \node[node](a) at (0,0.4) {a};
            \node[node](b) at (1.5,0.4) {b};
            \node[node](c) at (2.5,0) {c};
            %\node[node](d) at (1.2,-1) {d};
            \draw[bend left=15,->] (a) to node[above] {0} (b);
            \draw[bend left=15,->] (b) to node[below] {0} (a);
            %\draw (c) to node[below] {1} (d);
            \draw[loop,->] (c) to node[above] {2} (c);
          \end{tikzpicture}%
      ,
    \end{align*}
    and the clustering $C = \{\{a\},\{b\},\{c\}\}$.%
  \end{GraphsAsImages}
  $G'$ is a $C$-consistent improvement of $G$, because the weight of a between-cluster edge is decreased.
  %
  %Consider a graph $G$ with four nodes, $V=\{1,2,3,4\}$,
  %$\weight{12} = \weight{21} = 1$ and $\weight{34} = \weight{43} = 1$, all other edge weights are zero.
  %Take a clustering $C = \{\{1\},\{2\},\{3,4\}\}$.
  %This graph and clustering are shown in figure~\ref{fig:modularity-monotonicity}.
  %The total volume of $G$ is $4$, hence
  The modularity of $C$ in $G$ is
   %$\qmod(G,C) = -(1/4)^2 + -(1/4)^2 + 2/4 - (2/4)^2 = 1/8$.
   $\qmod(G,C) = 1/8$,
  while the modularity of $C$ in $G'$ is
  %
  %Decreasing the weight of $\weight{12}$ and $\weight{21}$ to zero is a $C$-consistent change, but it decreases the total volume to $2$.
  %The updated graph $G'$ has modularity
   %$\qmod(G',C) = 2/2 - (2/2)^2 = 0$.
   $\qmod(G',C) = 0$.
  So modularity can decrease with a consistent change of a graph, and hence it is not a monotonic quality function.
\end{proof}

Monotonicity might be too strong a condition.
When the goal is to find a clustering of a single graph, we are not actually interested in the absolute value of a quality function. Rather, what is of interest is the optimal clustering, and which changes to the graph preserve this optimum. At a smaller scaler, we can look at the relation between two clusterings. If $C$ is better then $D$ on a graph $G$, then on what other graphs is $C$ better then $D$?

We therefore define a relative version of monotonicity, in the hopes that modularity does satisfy this weaker version.

% \begin{definition}[Consistent worsening]
%   A graph $G$ is a $C$-consistent worsening of $G'$ if $G'$ is a $C$-consistent improvement of $G$.
% \end{axiom}

\begin{definition}[Relative monotonicity]
  A quality function $\Obj$ is \emph{relatively monotonic} if
  for all graphs $G$ and $G'$ and clusterings $C$ and $D$,
  %if $G'$ is a $C$-consistent improvement and a $D$-consistent worsening of $G$
  if $G'$ is a $C$-consistent improvement of $G$
  and $G$ is a $D$-consistent improvement of $G'$
  and $\Obj(G,C) \ge \Obj(G,D)$ then $\Obj(G',C) \ge \Obj(G',D)$.
\end{definition}
%The two conditions of consistent improvement mean that $

%Monotonicity of a quality function implies that it is relatively monotonic.
%But the converse is not always true.
%Because if the preconditions are met, then $\Obj(G',C) \ge \Obj(G,C) \ge \Obj(G,D) \ge \Obj(G',D)$.

%However, modularity also does not satisfy the relative version of monotonicity,

\begin{theorem}
  Modularity is not relatively monotonic.
\end{theorem}
\begin{proof}
  \begin{GraphsAsMatrices}%
    Take the graphs with the adjacency matrices
    \begin{align*}
      G = \begin{bmatrix}
            0 & 1 & 0 & 0 \\
            1 & 0 & 0 & 0 \\
            0 & 0 & 8 & 0 \\
            0 & 0 & 0 & 1 \\
          \end{bmatrix}
      & &
      G' = \begin{bmatrix}
            0 & 2 & 0 & 0 \\
            2 & 0 & 0 & 0 \\
            0 & 0 & 8 & 0 \\
            0 & 0 & 0 & 1 \\
           \end{bmatrix},
    \end{align*}
    and the clusterings $C=\{\{1,2,3\},\{4\}\}$ and $D=\{\{1\},\{2\},\{3,4\}\}$.
  \end{GraphsAsMatrices}%
  \begin{GraphsAsImages}%
    Take the graphs
    \begin{align*}
      G = \begin{tikzpicture}
            [baseline=3mm]
            \draw[use as bounding box] (-0.4,-0.5) rectangle (4.1,1.4);
            \node[node](a) at (0,  0.4) {a};
            \node[node](b) at (1.5,0.4) {b};
            \node[node](c) at (2.5,0) {c};
            \node[node](d) at (3.5,0) {d};
            \draw[bend left=15,->] (a) to node[above] {1} (b);
            \draw[bend left=15,->] (b) to node[below] {1} (a);
            %\draw (c) to node[below] {1} (d);
            \draw[loop,->] (c) to node[above] {8} (c);
            \draw[loop,->] (d) to node[above] {1} (d);
          \end{tikzpicture}
      &&
      G' = \begin{tikzpicture}
            [baseline=3mm]
            \draw[use as bounding box] (-0.4,-0.5) rectangle (4.1,1.4);
            \node[node](a) at (0,  0.4) {a};
            \node[node](b) at (1.5,0.4) {b};
            \node[node](c) at (2.5,0) {c};
            \node[node](d) at (3.5,0) {d};
            \draw[bend left=15,->] (a) to node[above] {2} (b);
            \draw[bend left=15,->] (b) to node[below] {2} (a);
            %\draw (c) to node[below] {1} (d);
            \draw[loop,->] (c) to node[above] {8} (c);
            \draw[loop,->] (d) to node[above] {1} (d);
          \end{tikzpicture}
        ,
%       \\
%       G = \begin{tikzpicture}
%             [baseline=7mm]
%             \draw[use as bounding box] (-0.6,-0.7) rectangle (2.1,2.4);
%             \node[node](a) at (0,0) {c};
%             \node[node](b) at (1.5,0) {d};
%             \node[node](c) at (0,1) {a};
%             \node[node](d) at (1.5,1) {b};
%             \draw[bend left=15,->] (a) to node[above] {1} (b);
%             \draw[bend left=15,->] (b) to node[below] {1} (a);
%             %\draw (c) to node[below] {1} (d);
%             \draw[loop,->] (c) to node[above] {1} (c);
%             \draw[loop,->] (d) to node[above] {8} (d);
%           \end{tikzpicture}
%       &&
%       G' = \begin{tikzpicture}
%             [baseline=7mm]
%             \useasboundingbox (-0.5,-0.5) rectangle (2,2.2);
%             \node[node](a) at (0,0) {c};
%             \node[node](b) at (1.5,0) {d};
%             \node[node](c) at (0,1) {a};
%             \node[node](d) at (1.5,1) {b};
%             \draw[bend left=15,->] (a) to node[above] {2} (b);
%             \draw[bend left=15,->] (b) to node[below] {2} (a);
%             %\draw (c) to node[below] {1} (d);
%             \draw[loop,->] (c) to node[above] {1} (c);
%             \draw[loop,->] (d) to node[above] {8} (d);
%           \end{tikzpicture},
    \end{align*}
    and the clusterings $C=\{\{a,b,c\},\{d\}\}$ and $D=\{\{a\},\{b\},\{c,d\}\}$.%
  \end{GraphsAsImages}
  $G'$ is a $C$-consistent improvement of $G$, because the weight of a within cluster edge is increased.
  $G$ is a $D$-consistent improvement of $G'$, because the weight of a between cluster edge is decreased.
  However
  $\qmod(G,C) = 20/121 > 16/121 = \qmod(G,D)$
  while
  $\qmod(G',C) = 24/169 < 28/121 = \qmod(G',D)$.
  This counterexample shows that modularity is not relatively monotonic.
\end{proof}

%------------------------------------------------------------------------------
\section{Adaptive Scale Modularity}\label{sec:qext}
%------------------------------------------------------------------------------

The problems with modularity stem from the fact that the total volume can change when changes are made to the graph.
It is therefore natural to look at a variant of modularity where the total volume is replaced by a constant $M$,
\begin{equation*}
  \qfixed(G,C) = \sum_{c \in C}\biggl( \frac{\within{c}}{M} - \Bigl( \frac{\vol{c}}{M} \Bigr)^2 \biggr)
  .
\end{equation*}
% This quality function is obviously consistent under unions, because
% \begin{equation*}
%   \qfixed(G_1\disjointunion G_2,C_1 \cup C_2) = \qfixed(G_1,C_1) + \qfixed(G_2,C_2).
% \end{equation*}
%This quality function is obviously local, but it is not scale invariant for a fixed value of $M$.
%However, fixed scale modularity is a scale invariant family.
This quality function is obviously local.
It is also a scale invariant family parameterized by $M$. However, this fixed scale modularity quality function is \emph{not} scale invariant for any fixed scale $M>0$.
%This quality function is obviously local.
% 
% This quality function is obviously not scale invariant. However, it is consistent under unions, because
% \begin{equation*}
%   \qfixed(G_1\disjointunion G_2,C_1 \cup C_2) = \qfixed(G_1,C_1) + \qfixed(G_2,C_2).
% \end{equation*}
% %

% \begin{theorem}
%   $M$-Modularity does satisfy consistency under unions, but not scale invariance.
% \end{theorem}
% \begin{proof}
%   TODO
% \end{proof}

We might hope that fixed scale modularity would be monotonic, because it doesn't suffer from the problem where changes in the edge weights affect the total volume. Unfortunately, fixed scale modularity has problems when the volume of a cluster starts to exceed $M/2$.
In that case, increasing the weight of within cluster edges starts to decrease the fixed scale modularity.
Looking at a cluster $c$ with volume $v_c = w_c + b_c$,
\begin{equation}
  \frac{\partial \qfixed(G,C)}{\partial w_c}
  %= 1/M - 2 v_c / M^2.
  = \frac{1}{M} - \frac{2 v_c }{M^2}.
\end{equation}
This derivative is negative when $2 v_c > M$, so in that case increasing the weight of a within-cluster edge will decrease the quality. Hence fixed scale modularity is not monotonic.

% 
% \begin{theorem}
%   $M$-Modularity is not monotonic.
% \end{theorem}
% \begin{proof}
%   An almost trivial counter example is given by the graph $G=(V,E)$ with $V=\{1\}$ and $e_{11} = 1$.
%   The $1$-modularity of the only possible clustering is $1-1^2 = 0$.
%   Since the self-loop is inside the cluster, increasing $e_{11}$ to $2$ is a consistent change.
%   The $1$-modularity of the changed graph is $2-2^2=-2$.
%   Therefore $M$-modularity is not monotonic.
%   %Take the graph $G=(V,E)$ with $V=\{1,2\}$ and $\weight{i}{j} = 1$ for all edges.
%   %This clustering $\{\{\}\}$ graph has $1$-modularity $1-$
% \end{proof}

The above argument also suggests a possible solution: add $2 v_c$ to the normalization factor $M$.
Or more generally, add $\vamount v_c$ with $\vamount \ge 2$,
which leads to the quality function
\begin{equation}
  %\qext = \sum_{c \in C}\biggl( \frac{\within{c}}{M + 2\vol{c}} - \Bigl( \frac{\vol{c}}{M + 2\vol{c}} \Bigr)^2 \biggr)
  \qext{M}{\gamma}(G,C) = \sum_{c \in C}\biggl( \frac{\within{c}}{M + \vamount\vol{c}} - \Bigl( \frac{\vol{c}}{M + \vamount\vol{c}} \Bigr)^2 \biggr)
  .
\end{equation}

% \emph{adaptive scale modularity} 

This \emph{adaptive scale modularity} quality function is clearly still permutation invariant, continuous and local.
For $M=0$ it is also scale invariant.
Since the value of $M$ should scale along with the edge weights, adaptive scale modularity is a scale invariant family parameterized by $M$.
Additionally, we have the following two theorems:
%The arguments for richness carries over from modularity, % TODO: does it??
%and the argument for consistency under unions carries over from $M$-modularity.

% \begin{theorem}
%   Adaptive scale modularity is monotonic for all $M \ge 0$.
%   \label{thm:ext-modularity-is-monotonic}
% \end{theorem}
% \begin{theorem}
%   Adaptive scale modularity is rich for all $M \ge 0$.
%   \label{thm:ext-modularity-is-rich}
% \end{theorem}
\begin{theorem}
  Adaptive scale modularity is rich for all $M \ge 0$ and $\vamount \ge 1$.
  \label{thm:ext-modularity-is-rich}
\end{theorem}
\begin{theorem}
  Adaptive scale modularity is monotonic for all $M \ge 0$ and $\vamount \ge 2$.
  \label{thm:ext-modularity-is-monotonic}
\end{theorem}
The proofs of these theorems can be found in appendices~\ref{sec:proof-ext-rich} and~\ref{sec:proof-ext-monotonic}.

This shows that adaptive scale modularity satisfies all six axioms we have defined for families of graph clustering quality functions, and the six axioms for single quality functions when $M=0$.
This shows that our extended set of axioms is consistent.

%------------------------------------------------------------------------------
\subsection{Relation to Other Quality Functions}\label{sec:qext-related}
%------------------------------------------------------------------------------

Interestingly, in the limit as $M$ goes to $0$, the adaptive-scale quality function becomes similar to normalized cut \citep{ShiMalik2000NormalizedCut} with an added constant,
\begin{align*}
  \qext{0}{\gamma}(G,C) = \frac{1}{\gamma}\sum_{c \in C} \Bigl( \frac{w_c}{v_c} - \frac{1}{\gamma} \Bigr).
\end{align*}
This $0$-adaptive modularity is also scale invariant as a single quality function.
%, and hence it satisfies all axioms we have defined.

Conversely, when $M$ goes to infinity the quality goes to $0$. However, the quality function approaches unnormalized cut in behavior:
\begin{align*}
  \lim_{M\to\infty} M\cdot \qext{M}{\gamma}(G,C) = \sum_{c \in C} w_c.
\end{align*}

This expression is similar to the Constant Potts model (CPM) by \citet{Traag2011ResolutionLimitScope},
\begin{equation}
  \qcpm(G,C) = \sum_{c \in C}\biggl( \within{c} - \gamma n_c^2 \biggr).
\end{equation}
In contrast to the quality functions discussed thus far, CPM uses the number of nodes instead of volume to control the size of clusters.
Like adaptive scale modularity, the constant Potts model satisfies all six axioms (as a family).

As stated before,
% the fixed scale and adaptive scale modularity quality functions are not scale invariant in general.
the fixed scale and adaptive scale modularity quality functions are a scale invariant family; they are not scale invariant for a fixed value of $M$ (except for $M=0$).
This is not a large problem in practice, since scale invariance is often sacrificed to overcome the resolution limit of modularity \citep{Fortunato2007ResolutionLimit}.
In fact, fixed scale modularity is proportional to the quality function introduced by \citet{Reichardt2004},
% Maybe: \citep{Arenas2008} %self loops
%Maybe: \citep{Lambiotte2010}
\begin{align*}
  Q_\text{RB}(G,C) %&= \sum_{i,j\in V} \Bigl(E(i,j) - \gamma_\text{RB} \frac{\vol{\{i\}}\vol{\{j\}}}{\vol{V}}\Bigr)\indicator[ i\sim_C j] \\
    = \sum_{c \in C} \Bigr( \within{c} - \gamma_\text{RB} \frac{\vol{c}^2}{\vol{V}}\Bigl) %\\
   %& = M \sum_{c \in C} \Bigr( \frac{\within{c}}{M} - \frac{\vol{c}^2}{M * \vol{V} / \gamma_\text{RB}}\Bigl) \\
    = M \cdot \qfixed(G,C),
\end{align*}
with $M = \vol{V} / \gamma_\text{RB}$.
% \begin{equation*}
%   \qfixed(G,C) = \frac{\vol{V}}{M} \sum_{c \in C}\biggl( \frac{\within{c}}{\vol{V}} - \gamma \Bigl( \frac{\vol{c}}{\vol{V}} \Bigr)^2 \biggr)
%   ,
% \end{equation*}
% with $\gamma = \vol{V}/M$.
%common choice of resolution

%------------------------------------------------------------------------------
\subsection{Parameter Dependence Analysis}\label{sec:qext-resolution}
%------------------------------------------------------------------------------

There has been a lot of interest in the so called resolution limit of modularity.%
% \citep{Fortunato2007ResolutionLimit}

This problem can be illustrated with a simple graph that consists of a ring of cliques, where each clique is connected to the next one with a single edge.
%TODO: See figure~\ref{}?
%  Or just see the reference!
%
We would like the clusters in the optimal clustering to correspond to the cliques in the ring.
It was observed by \citet{Fortunato2007ResolutionLimit} that, as the number of cliques in the ring increases, at some point the clustering with the highest modularity will have multiple cliques per cluster.

This resolution problem stems from the fact that the behavior of modularity depends on the total volume of the graph.
Both the fixed scale and adaptive scale modularity quality functions instead have a parameter $M$, and hence do not suffer from this problem.
In fact, any local quality function will not have a resolution limit in the sense of \citeauthor{Fortunato2007ResolutionLimit}.
A similar observation was made by \citet{Traag2011ResolutionLimitScope} in the context of modularity like quality functions.

%
% Consider a quality function with parameter that somehow controls the size of the clusters in the optimal clustering. 
% For certain settings of the parameter the two cliques will be joined into a single cluster, while for other settings the random subgraph will be split into multiple clusters. The question is whether there is a sweet spot where neither of these occur.
% \Citeauthor{Lancichinetti2011limits} showed that for modularity with a resolution control parameter, as the random subgraph becomes too large this sweet spot disappears.
% 
% We now do this same analysis for the adaptive scale modularity quality function.
% 
% Assume that each clique has within cluster volume $k$
% The quality for the two cliques clustered together is
% {\color{red}
% \begin{align*}
%   \qext{M}{\gamma}(\text{two-$k$-cliques-separated})
%     = 2\bigl( \frac{k(k-1)/2}{M + \gamma (k(k-1)/2+1)} )
%     - 2\bigl( \frac{k(k-1)/2+1}{M + \gamma (k(k-1)/2+1)} )^2
%    \\
%    =  2\bigl( \frac{k(k-1)/2}{M + \gamma (k(k-1)/2+1)} )
%     - 2\bigl( \frac{k(k-1)/2+1}{M + \gamma (k(k-1)/2+1)} )^2
% \end{align*}
% \begin{align*}
%   \qext{M}{\gamma}(\text{two-$k$-cliques-together})
%     = \bigl( \frac{2k(k-1)/2 + 2}{M + \gamma (2 k(k-1)/2+1)} )
%     - \bigl( \frac{k(k-1)/2+1}{M + \gamma (k(k-1)/2+1)} )^2
% \end{align*}
% }
% 
% The quality for two cliques with $k$ nodes

In real situations graphs are not uniform as in the ring-of-cliques model.
%
%More generally, we look at two subgraphs of varying sizes connected by a varying number of edges.
But we can still take simple uniform problems as a building block for larger and more complex graphs, since for local quality functions the rest of the network doesn't matter.
Therefore we will look at a simple problem with two subgraphs of varying sizes connected by a varying number of edges.
%These can be thought of as part of a larger network, but for local quality functions the rest of the network doesn't matter.
More precisely, we take two cliques each with within weight $w$, connected by edges with weight $b$. The total volume of this (sub)graph is then $2w+2b$.

There are three possible outcomes when clustering such a two-clique network: (1) the optimal solution has a single cluster; (2) the optimal solution has two clusters, corresponding to the two cliques; (3) the optimal solution has more than two clusters, splitting the cliques apart.
See Figure~\ref{fig:outcomes-illustration} for an illustration.
% TODO: figure? Does it really help? HowThe rest of this paper is organized as follows:
% Section~\ref{sec:definitions-notation} gives basic definitions.
% Next, section~\ref{sec:form-of-axioms} discusses different ways in which axioms could be formulated. to do it in BW
Which of these outcomes is desirable depends on the circumstances.

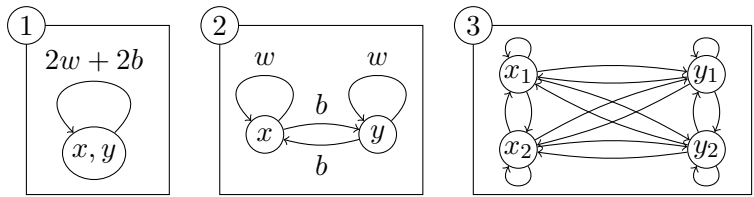
\begin{figure}[t]
  \centering
  \begin{tikzpicture}
    [baseline=3mm]
    \draw[use as bounding box] (-0.9,-0.8) rectangle (1.0,1.45);
    %\node[node] (a) at (0,-0.1) {$xy$};
    \node[ellipse,draw,inner sep=0,minimum size=7mm,minimum width=8mm] (a) at (0,-0.25) {$x,y$};
    \draw[looseness=6,min distance=4mm,->] (a) to node[above] {$2w+2b$} (a);
    \node[node,fill=white] at (-0.9,1.45) {1};
  \end{tikzpicture}
  \hspace*{4mm}
  \begin{tikzpicture}
    [baseline=3mm]
    \draw[use as bounding box] (-0.6,-0.8) rectangle (2.1,1.45);
    \node[node] (a) at (0,0) {$x$};
    \node[node] (b) at (1.5,0) {$y$};
    \draw[bend left=15,->] (a) to node[above] {$b$} (b);
    \draw[bend left=15,->] (b) to node[below] {$b$} (a);
    \draw[loop,->] (a) to node[above] {$w$} (a);
    \draw[loop,->] (b) to node[above] {$w$} (b);
    \node[node,fill=white] at (-0.6,1.45) {2};
  \end{tikzpicture}
  \hspace*{4mm}
  \begin{tikzpicture}
    [baseline=3mm]
    \draw[use as bounding box] (-0.6,-0.8) rectangle (3.1,1.45);
    \node[node,fill=white] at (-0.6,1.45) {3};
    \node[node] (a1) at (0,0-0.2) {$x_2$};
    \node[node] (a2) at (0,1-0.2) {$x_1$};
    \node[node] (b1) at (2.5,0-0.2) {$y_2$};
    \node[node] (b2) at (2.5,1-0.2) {$y_1$};
%     \draw[bend left=25,->] (a1) to node[left] {\footnotesize $w/4$} (a2);
%     \draw[bend left=25,->] (a2) to node[right] {\footnotesize $w/4$} (a1);
%     \draw[in=160,out=220,loop,->] (a1) to node[left] {\footnotesize $w/4$} (a1);
%     \draw[in=140,out=200,loop,->] (a2) to node[left] {\footnotesize $w/4$} (a2);
%     %
%     \draw[bend left=25,->] (b1) to node[left] {\footnotesize $w/4$} (b2);
%     \draw[bend left=25,->] (b2) to node[right] {\footnotesize $w/4$} (b1);
%     \draw[in=140-180,out=200-180,loop,->] (b1) to node[right] {\footnotesize $w/4$} (b1);
%     \draw[in=160-180,out=220-180,loop,->] (b2) to node[right] {\footnotesize $w/4$} (b2);
%     %
%     \draw[bend left=10,->] (a1) to node[above] {\footnotesize $\frac{b}4$} (b1);
%     \draw[bend left=10,->] (b1) to node[below] {\footnotesize $\frac{b}4$} (a1);
%     \draw[bend left=10,->] (a1) to node[above] {\footnotesize $\frac{b}4$} (b2);
%     \draw[bend left=10,->] (b1) to node[below] {\footnotesize $\frac{b}4$} (a2);
%     \draw[bend left=10,->] (a2) to node[above] {\footnotesize $\frac{b}4$} (b1);
%     \draw[bend left=10,->] (b2) to node[below] {\footnotesize $\frac{b}4$} (a1);
%     \draw[bend left=10,->] (a2) to node[above] {\footnotesize $\frac{b}4$} (b2);
%     \draw[bend left=10,->] (b2) to node[below] {\footnotesize $\frac{b}4$} (a2);
    %
    \draw[bend left=25,->] (a1) to node[left] {} (a2);
    \draw[bend left=25,->] (a2) to node[right] {} (a1);
    %\draw[in=160,out=220,loop,->] (a1) to node[left] {} (a1);
    %\draw[in=140,out=200,loop,->] (a2) to node[left] {} (a2);
    \draw[in=-90-30,out=-90+30,min distance=4mm,->] (a1) to node[left] {} (a1);
    \draw[in=90+30,out=90-30,min distance=4mm,->] (a2) to node[left] {} (a2);
    \draw[bend left=25,->] (b1) to node[left] {} (b2);
    \draw[bend left=25,->] (b2) to node[right] {} (b1);
%     \draw[in=140-180,out=200-180,loop,->] (b1) to node[right] {} (b1);
%     \draw[in=160-180,out=220-180,loop,->] (b2) to node[right] {} (b2);
    \draw[in=-90-30,out=-90+30,min distance=4mm,->] (b1) to node[left] {} (b1);
    \draw[in=90+30,out=90-30,min distance=4mm,->] (b2) to node[left] {} (b2);
    \draw[bend left=8,->] (a1) to node[above] {} (b1);
    \draw[bend left=8,->] (b1) to node[below] {} (a1);
    \draw[bend left=8,->] (a1) to node[above] {} (b2);
    \draw[bend left=8,->] (b1) to node[below] {} (a2);
    \draw[bend left=8,->] (a2) to node[above] {} (b1);
    \draw[bend left=8,->] (b2) to node[below] {} (a1);
    \draw[bend left=8,->] (a2) to node[above] {} (b2);
    \draw[bend left=8,->] (b2) to node[below] {} (a2);
  \end{tikzpicture}
  \caption{
    An illustration of the possible outcomes when clustering a two-clique network.
    %Groups of nodes are indicated by circles, clusters by boxes
    Clusters are indicated by circles.
    In outcome (3), the vertical edges each have weight $w/4$, while the horizontal and diagonal ones have weight $b/4$.
  }
  \label{fig:outcomes-illustration}
\end{figure}

Another heterogeneous resolution limit model was proposed by \citet{Lancichinetti2011limits}.
%A more complex case was proposed by \citet{Lancichinetti2011limits}.
In this situation there are two cliques of equal size connected by a single edge, and a random subgraph. Now the ideal solution would be to find three clusters, one for each clique and one for the random subgraph.
The optimal split of the random subgraph will roughly cut it in half, with a fixed fraction of the volume being between the two clusters \citep{Reichardt2007PartitioningAndModularity}.
So this model can be considered as a combination of two instances of our simpler problem, one for the two cliques and one for the random subgraph\footnote{\Citeauthor{Lancichinetti2011limits} include edges between the cliques and the random subgraph to ensure that the entire network is connected, these edges are not relevant to the problem}.
Hence, we want outcome (2) for the cliques, and outcome (1) for the random subgraph.
% TODO: meer zeggen
%
% TODO: put this somewhere?
%In general, when the $b \ge w$, the connection between the cliques is actually stronger than the internal connections. So the two  cliques can actually be thought of as one larger clique. Outcome (2) would correspond to splitting up this larger clique, i.e. outcome (3).

\begin{figure}[h]
  \centering
  \includegraphics[width=0.8\linewidth]{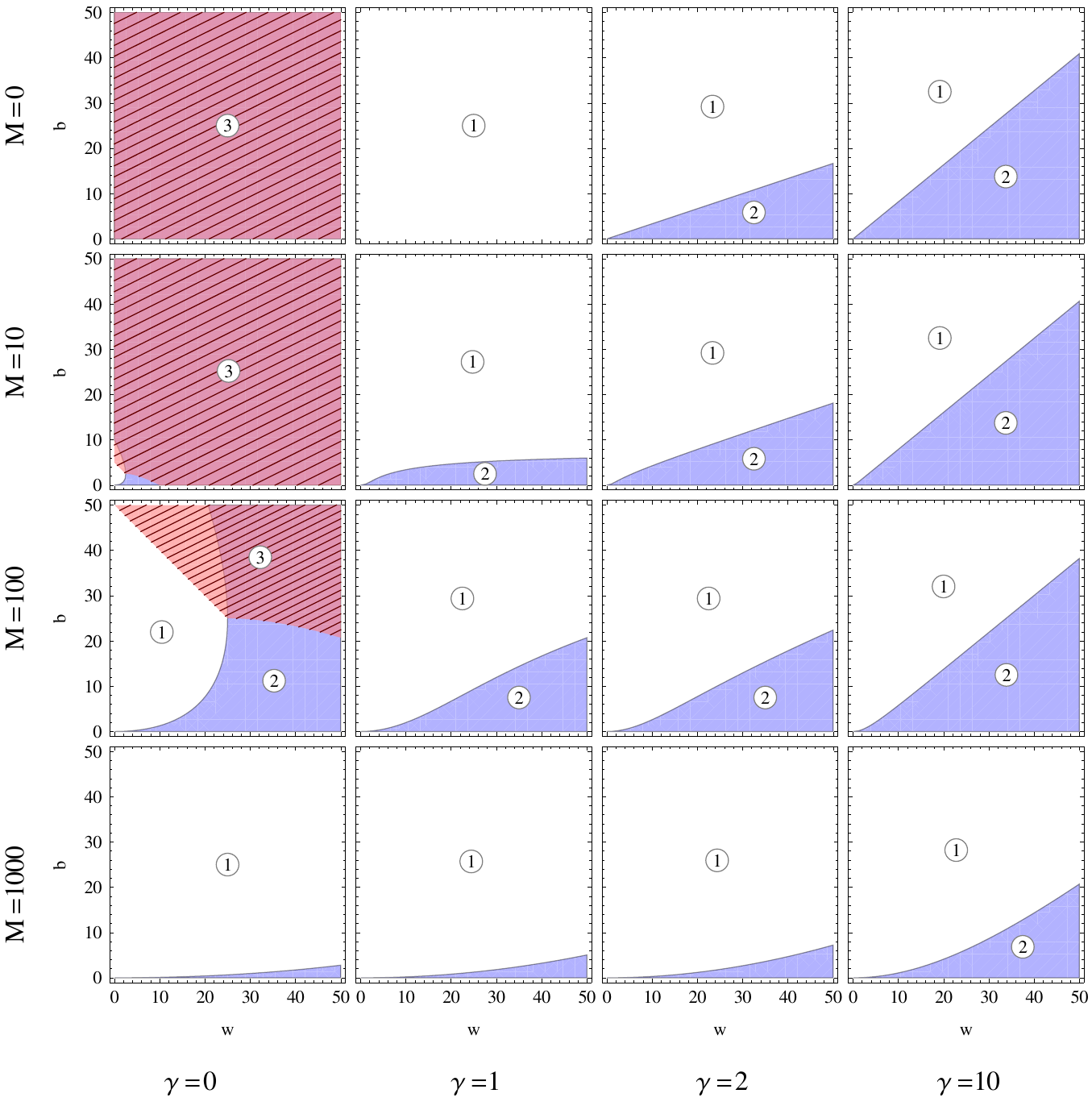}
  \caption{
    The behavior of $\qext{M}{\gamma}$ for varying parameter values.
    The graph consists of two subgraphs with $w$ internal weight each, connected by an edge with weigh $b$.
    Hence the volume of the total graph is $2w+2b$.
    In region (1) the optimal clustering has a single cluster,
    In region (2) (light blue) the optimal clustering separates the subgraphs.
    In region (3) (red, hatched) the subgraphs themselves will be split apart.
  }
  \label{fig:gext-wb}
\end{figure}

In Figure~\ref{fig:gext-wb} we show which graphs give which outcomes for adaptive scale modularity with various parameter settings.
The first column, $\gamma=0$, is of particular interest, since it corresponds to fixed scale modularity and hence also to $\Obj_\text{RB}$ and to modularity in certain graphs. In the third row we can see that when $2v=2w+2b > M=100$ the cliques are split apart. This is precisely the region in which monotonicity no longer holds.
Overall, the parameter $M$ has the effect of determining the scale; each row in this figure is merely the previous row magnified by a factor $10$. Increasing $M$ has the effect of merging small clusters. On the other hand, the $\gamma$ parameter controls the slope of the boundary between outcomes (1) and (2), i.e. the fraction of edges that should be within a cluster. This is most clearly seen when $M=0$, while otherwise the effect of $M$ dominates for small clusters.

% Section~\ref{sec:definitions-notation} gives basic definitions.
% Next, section~\ref{sec:form-of-axioms} discusses different ways in which axioms could be formulated.dary between regions (1) and (2), i.e. the fraction of edges that should be within a cluster. This is most clearly seen when $M=0$, while otherwise the effect of $M$ dominates for small clusters.

%In the third row we see that the probleThe rest of this paper is organized as follows:
% Section~\ref{sec:definitions-notation} gives basic definitions.
% Next, section~\ref{sec:form-of-axioms} discusses different ways in which axioms could be formulated.matic region (3) is a subset of $v = w+b > 50 = M/2$, which is also the point where monotonicity breaks down.
%
%The first row, $M=0$, corresponds to normalized cut

%
%, since the optimal split of the random subgraph will roughly cut it in half, with a fixed fraction of the volume being between the two clusters.

%==============================================================================
\section{Conclusion and Open Questions}
%==============================================================================

\begin{table}
  \newcommand{\yes}{$\checkmark$}
  \newcommand{\na}{n.a.}
  \newcommand{\no}{$-$}
  \begin{tabular}{lccccccc}
    & \rotatebox{90}{Permutation invariance} &
      \rotatebox{90}{Scale invariance} &
      \rotatebox{90}{Scale invariance (family)} & 
      \rotatebox{90}{Richness} &
      \rotatebox{90}{Monotonicity} &
      \rotatebox{90}{Locality} &
      \rotatebox{90}{Continuity}
    \\
    \hline
    Connected components & \yes & \yes  & \na  & \yes & \yes & \yes & \no \\
    Modularity      & \yes & \yes  & \na  & \yes & \no & \no & \yes \\
    \citet{Reichardt2004}       & \yes & \yes  & \yes & \yes & \no & \no & \yes \\
    Fixed scale
         modularity & \yes & $M=0$ & \yes & \yes & \no & \yes & \yes \\
    Adaptive scale
         modularity & \yes & $M=0$ & \yes & $\gamma \ge 1$ & $\gamma \ge 2$ & \yes & \yes \\
    Constant Potts Model \citep{Traag2011ResolutionLimitScope} & \yes & \no & \yes & $\gamma>0$ & \yes & \yes & \yes \\
    Normalized cut  & \yes & \yes  & \na  & \no  & \yes & \yes & \yes \\
  \end{tabular}
  \caption{Overview of quality functions discussed in this paper and the properties they satisfy.}
  \label{tbl:quality-functions-and-axioms}
\end{table}

In this paper we presented an axiomatic framework for graph clustering quality functions consisting of six properties. We showed that modularity does not satisfy the monotonicity property.
This motivated the derivation of a new family of quality functions, adaptive scale modularity, that satisfies all properties and has standard graph clustering quality functions as special cases.
%An overview of the discussed axioms and quality functions can be found in table~\ref{tbl:quality-functions-and-axioms}.
%
Results of an experimental parameter dependence analysis showed the high flexibility of adaptive scale modularity. 
However, adaptive scale modularity should not be considered the solution to all the problems of modularity, but rather an example of how axioms can be used in practice. % Do we want to say this?

An overview of the discussed axioms and quality functions can be found in table~\ref{tbl:quality-functions-and-axioms}.
Many more quality functions have been proposed in the literature, so this list is by no means exhaustive.
An interesting topic for future research is to make
 a survey of which existing quality functions satisfy which of the proposed properties.
% For instance, quality functions that do not satify locality are those based on centrality. Indeed, centrality based clustering has been shown to be rather unstable (see e.g. \citealp{LandherrFH10}), precisely because it is not local. 

We also investigated resolution-limit-free quality functions as defined in \citep{Traag2011ResolutionLimitScope}.
As illustrated in section \ref{sec:qext-resolution}, adaptive scale modularity  allows to perform clustering  at various resolutions, by varying the values of its two parameters. However it is not resolution-limit-free. 
%In particular, we showed that the constant Potts model resolution-limit-free quality function therein introduced satisfies all axioms. % T: I'd rather not get too much off-topic

Our paper did not address questions such as finding a best quality function \citep*{Almeida2011}, or selecting a significant resolution scale \citep{Traag2013significance}. The aim  was to provide necessary conditions about what a good quality function is, in order to rule out and/or to improve quality functions. The  proposed axioms and the introduction of adaptive scale modularity are an effort in this direction.

We also did not address the question of finding a clustering with the highest quality.
Finding the optimal value of quality functions such as modularity is NP-hard \citep{Brandes2008modularityNPHard}, but several heuristic and approximation algorithms have been developed. One class of algorithms uses a divisive approach, see for instance \citet{Newman2006spectral,RuanZhang2008HQCut}.
%First a good split into a few clusters is found, and then the subgraphs for each of these clusters are further subdivided independently.
%For this tactic to be valid, an optimal or close to optimal clustering on the sub-problems should also be a near optimal clustering of the entire graph.
For such a tactic to be valid, an optimal or close to optimal clustering of a subgraph should also be a near optimal clustering of the entire graph.
This is ensured by locality. Recently \citet{DinhT13} proposed polynomial-time approximation algorithms for the modularity maximization in the context of scale free networks. 
It would be interesting to investigate the suitability of these algorithms for adaptive scale modularity maximization.
%TODO: OPEN QUESTIONS:

%EXTENSION TO DIRECTED GRAPHS, OVERLAPPING CLUSTERS?
In this work we have only looked at non-negative weights, undirected graphs, and only at hard partitioning.
An extension to graphs with negative weights, to directed graphs and to overlapping clusters remains to be investigated.
Another open problem is how to use these axioms for reasoning about quality functions and clustering algorithms.
%TODO: USE OF AXIOMS FOR REASONING ABOUT CLUSTERINGS AND DEVELOPING CLUSTERING METHODS?
% Do we need to say more?

%==============================================================================
\section*{Acknowledgments}
%==============================================================================
We thank the reviewers for their comments.
This work has been partially funded by the Netherlands Organization for Scientific Research (NWO) within the NWO project 612.066.927.

%==============================================================================
\appendix
%\section{Proofs}
%==============================================================================

%------------------------------------------------------------------------------
\section{Proof of Theorem~\ref{thm:modularity-is-rich} (Modularity is Rich)}
\label{sec:proof-modularity-rich}
%------------------------------------------------------------------------------

The proofs of richness rely on clique graphs,

\begin{definition}[Clique graph]
  Let $V$ be a set of nodes, $C$ be a partition of $V$, and $k$ be a positive constant.
  The \emph{clique graph of $C$ with edge weight $k$} is defined as
  $G = (V,E)$ where $E(i,j) = k$ if $i \sim_C j$ and $E(i,j) = 0$ otherwise.
  %E(i,j) = k \mathbf{1}[i \sim_C j]$.
\end{definition}

% \begin{theorem}
%   Modularity is rich.
% \end{theorem}
\begin{proof}%[Proof of theorem~\ref{thm:modularity-is-rich} (richness of modularity)]
  
  Let $V$ be a set of nodes and $C\neq\{V\}$ be a clustering of $V$.
  % $C$ is not the trivial clustering $\{V\}$.
  %So $C$ has at least two clusters.
  %Let $V$ be a set of nodes and $C$ be a clustering of $V$.
  %TODO %TODO ^^
  %
  %Now define a graph with $e_{ii} = |V|$, $\weight{i}{j} = 1$ if $i \neq j$ and $i \sim_C j$, and $\weight{i}{j} = 0$ otherwise.
  Let $G=(V,E)$ be a clique graph of $C$ with edge weight $1$.
  Note that $\weight{i}{i} = 1$, so any possible cluster will have a positive volume.
  Let $D$ be a clustering of $G$ with maximal modularity.
  
  Suppose that there is a cluster $d \in D$ that contains $i,j \in d$ with $i \not\sim_C j$.
  Then we can split the cluster into $d_1 = \{k \in d \mid k \sim_C i\}$ and $d_2 = \{k \in d \mid k \not\sim_C i\}$.
  Because there are no edges between nodes in $d_1$ and nodes in $d_2$,
  it is the case that $w_d = w_{d_1} + w_{d_2}$.
  Both $d_1$ and $d_2$ are non-empty and have a positive volume, so
   $v_d^2 = (v_{d_1} + v_{d_2})^2 < v_{d_1}^2 + v_{d_2}^2$.
  Therefore $\qmod(G,D) < \qmod(G,D\setminus \{d\} \cup \{d_1,d_2\})$.
  So $D$ does not have maximal modularity, which is a contradiction.
  
  Suppose, on the other hand that all clusters $d \in D$ are a subset of some cluster in $C$, i.e. $D$ is a refinement of $C$.
  Then either $D = C$,
   or there are two clusters $d_1,d_2 \in D$ that are both a subset of the same cluster $c \in C$.
  In the latter case we can combine the two clusters into $d=d_1\cup d_2$.
  The within weight of this combined cluster is $w_d = |d|^2=w_{d_1}+w_{d_2}+2|d_1||d_2|$.
  The squared volume of the combined cluster is $v_d^2 = |d|^2|c|^2 = v_{d_1}^2 + v_{d_2}^2 + 2|d_1||d_2||c|^2$.
  So this changes increases the modularity by
  \begin{align*}
    &\qmod(G,D\setminus \{d_1,d_2\} \cup \{d\}) - \qmod(G,D)
      \\&\quad= 2 |d_1||d_2| / v_V
        - 2 |d_1||d_2| |c|^2 / v_V^2
      \\&\quad= 2 |d_1||d_2| (v_V - |c|^2) / v_V^2
      > 0,
  \end{align*}
  which contradicts the assumption that $D$ has maximal modularity.
  Therefore the only optimal clustering of $G$ is $C$.
  Note that the above inequality only holds when $|c|^2 = v_c < v_V$, which is the case because $C \neq \{V\}$.
  
  %-----------
  
  When $C=\{V\}$, a clique graph will not work; because both $\{V\}$ and the clustering that assigns half the nodes to one cluster, and half to another have modularity equal to $0$.
  In this case, instead define $G=(V,E)$ by $E(i,j) = 1$ if $i \neq j$ and $0$ if $i = j$.
  Then the modularity for $C$ is $q(G,\{V\}) = 0$.
  Any cluster $d$ in a clustering $D$ will have $v_d = |d|(|V|-1)$ and $w_d = |d|(|d|-1)$.
  Therefore the contribution of this cluster to the total quality is $-|d|(|V|-|d|)/(|V|^2(|V|-1))$, which is negative when $|d|<|V|$.
  So the modularity of any clustering other than $\{V\}$ will be negative, hence $\{V\}$ is the only optimal clustering.
  %For any other clustering $D$, all clusters $d$ will have $w_d = |d|(|d|-1) < |d|(|V|-1) = v_d$.
  %Hence
  
  Since for every $C$ we can construct a graph where $C$ is the only optimal clustering, modularity is rich.
  
\end{proof}

%------------------------------------------------------------------------------
\section{Proof of Theorem~\ref{thm:ext-modularity-is-rich} (Adaptive Scale Modularity is Rich)}
\label{sec:proof-ext-rich}
%------------------------------------------------------------------------------

%The easiest case for proving the richness of adaptive scale modularity $\qext{M}{\gamma}$ is when $M=0$,
%because in that case the quality for a cluster can be expressed in terms of $w_d/v_d$.
%The proof for the general case has two parts: first we

Denote by $\maxfrac{C}(d)$ the largest fraction of any cluster from $C$ that is contained in a cluster $d$.
\begin{align*}
  %\maxfrac{C}(d) = \max_{c \in C} |c \cap d|/|c|.
  \maxfrac{C}(d) = \max_{c \in C} \frac{|c \cap d|}{|c|}.
\end{align*}
%and let $\maxfrac{C}(D)
For any clustering $D$ we have that
\begin{align}
  \sum_{d \in D} \maxfrac{C}(d)
  % \le |C|
  = \sum_{d \in D} \max_{c \in C} \frac{|c \cap d|}{|c|}
  \le \sum_{d \in D} \sum_{c \in C} \frac{|c \cap d|}{|c|}
  %= \sum_{c \in C} \sum_{d \in D} |c \cap d|/|c|
  %= \sum_{c \in C} 1
  = |C|.
  \label{eq:frac_c}
\end{align}
And since $\maxfrac{C}(d) \le 1$ for all clusters $d$, we also have that
\begin{align}
  \sum_{d \in D} \maxfrac{C}(d) \le |D|.
  \label{eq:frac_d}
\end{align}

\begin{lemma}
  \label{lem:bound-frac}
  For a clique graph of $C$ it is the case that $w_d / v_d \le \maxfrac{C}(d)$.
  % with equality only when $d$ is the union of one or more clusters from $C$.
\end{lemma}
\begin{proof}
  Given a cluster $d$ and a clique graph $G$ of $C$ with weight $k>0$,
  the volume of $d$ is
  \[v_d = \sum_{c \in C} k |c \cap d||c|,\]
  and the within cluster weight is
  \[w_d = \sum_{c \in C} k |c \cap d|^2.\]
  Therefore
  \begin{align*}
    w_d
    \le \sum_{c \in C} k |c \cap d||c|\maxfrac{C}(d)
    = v_d \maxfrac{C}(d).
  \end{align*}
  And hence
  $w_d / v_d \le \maxfrac{C}(d)$.
\end{proof}

\begin{lemma}
  Let $G$ be the clique graph of a clustering $C$ with weight $k$, and let $0 < \beta < 1$ be a constant.
  Then $\sum_{d \in D} (w_d / v_d - \beta) = (1 - \beta)|C|$ if $D = C$,
  while $\sum_{d \in D} (w_d / v_d - \beta) < (1 - \beta)|C| - \epsilon$ if $D \neq C$,
  where $\epsilon = \min(\beta,1-\beta,1/|V|)/2$.
  \label{lem:ext-modularity-zero-rich}
\end{lemma}
\begin{proof}
  Suppose that $D = C$, then for every cluster $c \in C$, $w_c = v_c = k|c|^2$, and so
  \begin{equation*}
     \sum_{c \in C} \Bigl(\frac{w_d}{v_d} - \beta \Bigr) = (1 - \beta)|C|.
  \end{equation*}
  
  Otherwise, $D \neq C$.
  Assume that $\sum_{d \in D} (w_d / v_d - \beta) \ge (1 - \beta)|C| - \min(\beta,1/|V|)/2$.
  By Lemma~\ref{lem:bound-frac},
  \begin{align*}
       &   |C| - \beta(|C| + 1)
    \\<&   |C| - \beta|C| - \epsilon
    \\\le& \sum_{d \in D} (\frac{w_d}{v_d} - \beta)
    \\\le& \sum_{d \in D} (\maxfrac{C}(d) - \beta)
    \\\le& |C| - \beta|D|.
  \end{align*}
  Since $\beta > 0$, this implies that $|D| < |C| + 1$.
  
  Additionally, since $\maxfrac{C}(d) \le 1$ for all clusters $d \in D$,
  \begin{align*}
         & (1 - \beta)(|C| - 1)
    \\ < & (1 - \beta)|C| - \epsilon
    \\\le& \sum_{d \in D} (\maxfrac{C}(d) - \beta)
    \\\le& (1 - \beta)|D|
  \end{align*}
  Since $\beta < 1$, this implies that $|D| > |C| - 1$.
  Hence $|D| = |C|$.
  
  Suppose that $\maxfrac{C}(d) < 1$ for some $d \in D$, which implies that $|c \cap d|<|c|$.
  Because edges are discrete, this can only happen when $|c \cap d|\le |c|-1$ for all clusters $c$. And the size of clusters is bounded by $|c| \le |V|$. Hence $\maxfrac{C}(d) \le (|V|-1)/|V| = 1 - 1/|V|$.
  And since for all other clusters $d'$, $\maxfrac{C}(d') \le 1$, we then have
  \begin{align*}
         & \sum_{d \in D} (\maxfrac{C}(d) - \beta)
    \\\le& (1 - \beta)|D| - 1/|V|
    \\<  & (1 - \beta)|C| - \epsilon
    \\\le& \sum_{d \in D} (w_d / v_d - \beta)
    \\\le& \sum_{d \in D} (\maxfrac{C}(d) - \beta),
  \end{align*}
  which is a contradiction.
  Hence, it must be the case that $\maxfrac{C}(d) = 1$ for all clusters $d \in D$.
  By the definition of $\maxfrac{C}$ this means that for every $d$ there is a cluster $c \in C$ such that $|c \cap d|=|c|$, and therefore $c \subseteq d$. Since the clusters are disjoint and $|D|=|C|$, this implies that $D = C$.
  Which is a contradiction, so $\sum_{d \in D} (w_d / v_d - \beta) < (1 - \beta)|C| - \epsilon$.
\end{proof}

When $M=0$, the adaptive scale modularity reduces to $w_d/(\gamma v_d) - |D|/\gamma^2$, and the above lemma is enough to prove richness. For non-zero values of $M$, we can get `close enough' by choosing large enough edge weights. This is formalized in the following lemma.

\begin{lemma}
  \label{lem:ext-modularity-bound}
  Let $d$ be a cluster in a clustering of a clique graph of $C$ with weight $k$.
  %And let $\gamma \ge 2$.
  Then
  \begin{align*}
      \frac{w_d}{v_d} - \beta - \beta M / k
    \le
      q(d)/\beta
    \le
      \frac{w_d}{v_d} - \beta + 2 \beta^2 M / k
    ,
  \end{align*}
  where
  \[
   q(d) = \frac{w_d}{M + v_d/\beta} - \Bigl(\frac{v_d}{M + v_d/\beta}\Bigr)^2
  \] denotes the contribution of $d$ to the $M$-adaptive modularity.
\end{lemma}
\begin{proof}
  Since clusters are non-empty, and in a clique graph $\weight{i}{i} = k$,
  it follows that $v_d \ge w_d \ge k$.
  So
  \begin{align*}
      &q(d)/\beta
    %=
    %  \frac{(\beta M + v_d) w_d - \beta v_d^2}{(\beta M + v_d)^2}
    \\=&
      \frac{\beta M w_d + v_d w_d - \beta v_d^2}{(\beta M + v_d)^2}
    %=
    %  \frac{\beta M w_d + (w_d - \beta v_d) v_d}{(\beta M + v_d)^2}
    %=
    %  \frac{\beta M w_d + (w_d - \beta v_d) v_d (\beta M + v_d)^2 / (\beta M + v_d)^2}{(\beta M + v_d)^2}
    %=
    %  \frac{w_d}{v_d} - \beta + \frac{\beta^3 M^2 v_d + 2 \beta^2 M v_d^2 - \beta^2 M^2 w_d - \beta M v_d w_d}
    %                                 {v_d (\beta M + v_d)^2}
    %=
    %  \frac{w_d}{v_d} - \beta + \frac{\beta^3 M^2 + 2 \beta^2 M v_d - \beta^2 M^2 w_d/v_d - \beta M w_d}
    %                                 {(\beta M + v_d)^2}
    \\=&
      \frac{w_d}{v_d} - \beta + \frac{\beta^2 M (\beta M + 2 v_d) - \beta^2 M^2 w_d/v_d - \beta M w_d}
                                     {(\beta M + v_d)^2}
%     \le
%       %\frac{w_d}{v_d} - \beta + \frac{\beta^3 M^2 + 2 \beta^2 M v_d}{(\beta M + v_d)^2}
%       \frac{w_d}{v_d} - \beta + \frac{\beta^2 M (\beta M + 2 v_d)}{(\beta M + v_d)^2}
%     \le
%       %\frac{w_d}{v_d} - \beta + \frac{2 \beta^3 M^2 v_d + 4 \beta^2 M v_d^2}{\beta M v_d + 2 v_d^2}
%       \frac{w_d}{v_d} - \beta + \frac{2 \beta^2 M (\beta M + 2 v_d)}{\beta M v_d + 2 v_d^2}
%     =
%       \frac{w_d}{v_d} - \beta + \frac{2 \beta^2 M}{v_d}
%     \le
%       \frac{w_d}{v_d} - \beta + \frac{2 \beta^2 M}{k}
%       
    \\\le&
      \frac{w_d}{v_d} - \beta + \frac{\beta^2 M (\beta M + 2 v_d)}{(\beta M + v_d)^2}
    \\\le&
      \frac{w_d}{v_d} - \beta + \frac{2 \beta^2 M (\beta M + 2 v_d)}{(\beta M + v_d)(\beta M + 2 v_d)}
    \\=&
      \frac{w_d}{v_d} - \beta + \frac{2 \beta^2 M}{\beta M + v_d}
    \\\le&
      \frac{w_d}{v_d} - \beta + \frac{2 \beta^2 M}{k}.
  \end{align*}
  And
  since $w_d \le v_d$,
  \begin{align*}
      &q(d)/\beta
    \\=&
      \frac{w_d}{v_d} - \beta + \frac{\beta^2 M (\beta M + 2 v_d) - \beta^2 M^2 w_d/v_d - \beta M w_d}
                                     {(\beta M + v_d)^2}
    \\\ge& % removing positive terms, negative w -> v
      \frac{w_d}{v_d} - \beta - \frac{\beta^2 M^2 + \beta M v_d}
                                     {(\beta M + v_d)^2}
    \\=&
      \frac{w_d}{v_d} - \beta - \frac{\beta M}{\beta M + v_d}
    \\\ge& % increase numerator of negative
      \frac{w_d}{v_d} - \beta - \frac{\beta M}{k}.
  \end{align*}
\end{proof}

Combining these lemmas yields the proof of the general theorem:
\begin{proof}
  Given a clustering $C$.
  Define $\beta = 1/\gamma$. If $\gamma > 1$ then $0 < \beta < 1$.
  Pick $k > 3 |V|\beta^2 M / \epsilon$
  where $\epsilon$ is defined as in Lemma~\ref{lem:ext-modularity-zero-rich}.
  %This implies that $2|V|\beta^2M/k - \epsilon < |V|\beta^2M/k$.
  
  Let $G$ be the clique graph of $C$ with weight $k$.
  Let $D \neq C$ be a clustering of $G$.
  Then by Lemmas~\ref{lem:ext-modularity-zero-rich} and~\ref{lem:ext-modularity-bound},
  \begin{align*}
     &\qext{M}{\gamma}(G,D)/\beta
     \\
        =& \sum_{d \in D} q(d)
     \\
       \le& \sum_{d \in D} (w_d/v_d - \beta + 2\beta^3M/k)
     \\
       \le& (1-\beta)|C| + 2|D|\beta^3M/k - \epsilon
     \\
       \le& (1-\beta)|C| + 2|V|\beta^2M/k - \epsilon
     \\
       <  & (1-\beta)|C| -  |V|\beta^2M/k
     \\
       \le& (1-\beta)|C| -  |C|\beta^2M/k
     \\
         =& \sum_{c \in C} (w_c/v_c - \beta + \beta^2M/k)
     \\
       \le& \qext{M}{\gamma}(C)/\beta.
  \end{align*}
  Hence the quality is maximal for $C$.
  Since there is a clique graph and $k$ for every clustering, adaptive scale modularity is rich.
\end{proof}

%------------------------------------------------------------------------------
\section{Proof of Theorem~\ref{thm:ext-modularity-is-monotonic} (Adaptive Scale Modularity is Monotonic)}
\label{sec:proof-ext-monotonic}
%------------------------------------------------------------------------------

\begin{proof}

  Given a constants $M > 0$ and $\vamount \ge 2$, a graph $G$ and a clustering $C$ of $G$.
  Let $c \in C$ be any cluster.
  %Let $G$ be a graph and let $c \in C$ be any cluster in a clustering of $G$.
  %Let $M$ be a positive constant.
  Writing the volume of $c$ as $v_c = w_c + b_c$,
  the contribution of this cluster to the quality of $G$ is $\objpart(w_c,b_c)$ where
  \begin{align*}
    \objpart(w,b) &= \frac{w}{M + \vamount w + \vamount b} - \Bigl( \frac{w + b}{M + \vamount w + \vamount b} \Bigr)^2
    %\\ &=\frac{w M + w^2 - b^2}{(M + 2w + 2b)^2}
    .
  \end{align*}
  The partial derivatives of $\objpart$ are
  \begin{align*}
    \frac{\partial \objpart(w,b)}{\partial w}
    &= \frac{M^2 + (\vamount-2)M(w + b) + \vamount b (M + \vamount w + \vamount b)}{(M + \vamount w + \vamount b)^3}
    \ge 0
    \\
    \frac{\partial \objpart(w,b)}{\partial b}
    &= - \frac{\vamount w M + (w + b) (M + \vamount^2 w)}{(M + \vamount w + \vamount b)^3}
    \le 0.
  \end{align*}
  This means that $\objpart$ is a monotonically non-decreasing function in $w$ and a non-increasing function in $b$.
  
  For any graph $G'$ that is a $C$-consistent change of $G$, it holds that $w'_c \ge w_c$ and $b'_c \le b_c$.
  So $\objpart(w'_c,b'_c) \ge \objpart(w_c,b_c)$.
  And therefore $\qext{M}{\gamma}(G',C) \ge \qext{M}{\gamma}(G,C)$.
  So adaptive scale modularity is monotonic.
\end{proof}

%==============================================================================
%\section{References}
%==============================================================================

%\bibliographystyle{abbrvnat}
\bibliography{clustering}

\end{document}